\documentclass[submission,copyright,creativecommons]{eptcs}
\providecommand{\event}{TARK 2023} %

\usepackage{iftex}

\ifpdf
  \usepackage{underscore}         %
  \usepackage[T1]{fontenc}        %
\else
  \usepackage{breakurl}           %
\fi

\title{A Theory of Bounded Inductive Rationality}

\author{Caspar Oesterheld
\institute{Computer Science Department\\
Carnegie Mellon University\\
Pittsburgh, PA, USA}
\email{oesterheld@cmu.edu}
\and
Abram Demski
\institute{Machine Intelligence Research Institute\\
Berkeley, California, USA}
\and 
Vincent Conitzer
\institute{Computer Science Department\\
Carnegie Mellon University\\
Pittsburgh, PA, USA}
}
\def\titlerunning{A Theory of Bounded Inductive Rationality}
\def\authorrunning{C. Oesterheld, A. Demski \& V. Conitzer}

\usepackage{xparse}
\NewDocumentCommand{\foocmd}{ O{default1} O{default2} m }{#1~#2~#3}

\NewDocumentCommand{\citep}{ O{} O{} m }{(\ifthenelse{\equal{#1}{}}{}{#1~}\cite{#3}\ifthenelse{\equal{#2}{}}{}{, #2})}
\NewDocumentCommand{\citeyear}{O{} m }{\cite{#2}\ifthenelse{\equal{#1}{}}{}{, #1}}
\NewDocumentCommand{\citealp}{ O{} O{} m }{\ifthenelse{\equal{#1}{}}{}{#1~}\cite{#3}\ifthenelse{\equal{#1}{}}{}{, #2}}

\usepackage{amsmath}
\usepackage{amsthm}
\usepackage{thmtools} 
\usepackage{thm-restate}
\usepackage{amsfonts}
\usepackage{nicefrac}
\usepackage{mathtools}

\newtheorem{definition}{Definition}
\newtheorem{conjecture}{Conjecture}
\newtheorem{theorem}{Theorem}
\newtheorem{lemma}[theorem]{Lemma}
\newtheorem{proposition}[theorem]{Proposition}

\usepackage{multirow,array}

\newcolumntype{L}[1]{>{\raggedright\let\newline\\\arraybackslash\hspace{0pt}}m{#1}}
\newcolumntype{C}[1]{>{\centering\let\newline\\\arraybackslash\hspace{0pt}}m{#1}}
\newcolumntype{R}[1]{>{\raggedleft\let\newline\\\arraybackslash\hspace{0pt}}m{#1}}

\newcommand{\BibTeX}{\rm B\kern-.05em{\sc i\kern-.025em b}\kern-.08em\TeX}

\usepackage{cleveref}

\Crefname{section}{Sect.}{Sects.}

\newcommand{\DP}{\mathrm{DP}}

\usepackage{ifthen}

\newboolean{extendedversion}
\setboolean{extendedversion}{false}
\usepackage{environ}
\NewEnviron{extendedonlyblock}
  {\ifthenelse{\boolean{extendedversion}}{\BODY}{}}
\newcommand{\extendedonlybit}[1]{\ifthenelse{\boolean{extendedversion}}{#1}{}}
\NewEnviron{abridgedonlyblock}
  {\ifthenelse{\not{\boolean{extendedversion}}}{\BODY}{}}
\newcommand{\abridgedonlybit}[1]{\ifthenelse{\boolean{extendedversion}}{}{#1}}

\usepackage{nicefrac}

\usepackage{thm-restate}

\usepackage{csquotes}

\DeclareMathOperator*{\argmax}{arg\,max}
\DeclareMathOperator*{\argmin}{arg\,min}

\usepackage{verbatim}

\usepackage{xcolor}

\newboolean{commentsactivated}
\setboolean{commentsactivated}{false}

\usepackage{bbm}

\DeclareMathOperator{\Var}{Var}
\DeclareMathOperator*{\avg}{avg}

\begin{document}

\maketitle 

\begin{abstract}
The dominant theories of rational choice assume %
logical omniscience. That is, they assume that when facing a decision problem, an agent can perform all relevant computations and determine the truth value of all relevant logical/mathematical claims. This assumption is unrealistic when, for example, we offer bets on remote digits of $\pi$ or when an agent faces a computationally intractable planning problem.
Furthermore, the assumption of logical omniscience creates contradictions in cases where the environment can contain descriptions of the agent itself.
Importantly, strategic interactions as studied in game theory are decision problems in which a rational agent is predicted by its environment (the other players).
In this paper, we develop a theory of rational decision making that does not assume logical omniscience. 
We consider agents who repeatedly face decision problems (including ones like betting on digits of $\pi$ or games against other agents). The main contribution of this paper is to provide a sensible theory of rationality for such agents.
Roughly, we require that a boundedly rational inductive agent tests each efficiently computable hypothesis infinitely often and follows those hypotheses that keep their promises of high rewards. We then prove that agents that are rational in this sense have other desirable properties. For example, they learn to value random and pseudo-random lotteries at their expected reward. Finally, we consider strategic interactions between different agents and prove a folk theorem for what strategies bounded rational inductive agents can converge to.
\end{abstract}

\section{Introduction}
\label{sec:introduction}

The dominant theories of rational decision making -- in particular Bayesian theories -- assume logical omniscience, i.e., that rational agents can determine the truth value of any relevant logical statement. In some types of decision problems, this prevents one from deriving any recommendation from these theories, which is unsatisfactory (\Cref{sec:log-unc-strategic-interactions}). For one, there are problems in which computing an optimal choice is simply computationally intractable. For example, many planning problems are intractable. Second, the assumption of logical omniscience creates contradictions (resembling classic paradoxes of self reference, such as the liar's paradox) if the environment is allowed to contain references to the agent itself. These issues arise most naturally when multiple rational agents interact and reason about one another.

This paper develops a novel theory of boundedly rational inductive agents (BRIAs) that does not assume logical omniscience and yields sensible recommendations in problems such as the ones described above. Rather than describing how an agent should deal with an individual decision, the theory considers how an agent learns to choose on a sequence of different decision problems. 
We describe the setting in more detail in \Cref{sec:setting}.

The core of our theory is a normative rationality criterion for such learning agents. Roughly, the criterion requires that a boundedly rational inductive agent test each efficiently computable hypothesis (or more generally each hypothesis in some class) infinitely often and follows hypotheses that keep their promises of high rewards. We describe the criterion in detail in \Cref{sec:the-criterion}. Importantly, the criterion can be satisfied by computationally bounded agents, as we show in \Cref{sec:computing-BRIAs}.

We demonstrate the appeal of our criterion by showing that it implies desirable and general behavioral patterns. In \Cref{sec:easy-options}, we show that on sequences of decision problems in which one available option guarantees a payoff of at least $l$, BRIAs learn to obtain a reward of at least $l$. Thus, in particular, they avoid Dutch books (in the limit).  %
We further show that similarly on sequences of decision problems in which one available option pays off truly or algorithmically randomly with mean $\mu$, BRIAs learn to obtain a reward of at least $\mu$.
Finally, we consider decision problems in which one BRIA plays a strategic game against another BRIA. We show that BRIAs can converge to any individually rational correlated strategy profile. BRIAs are thus a promising model for studying ideas such as superrationality (i.e., cooperation in the one-shot Prisoner's Dilemma) \cite{Hofstadter1983} (cf.~\Cref{sec:rel-work-dt-Newcomb}).
Related work is discussed in \Cref{sec:rel-work}.
Throughout this paper, we describe the key ideas for our proofs in the main text. Detailed proofs are given in \Cref{appendix:proofs}. %

\section{Setting}
\label{sec:setting}

Informally, we consider an agent who makes decisions in discrete time steps. At each time step she faces some set of available options to choose from. She selects one of options and receives a reward. She then faces a new decision problem, and so on.

Formally, let $\mathcal{T}$ be some language describing available \textit{options}. A \textit{decision problem} $\DP\in \mathrm{Fin}(\mathcal{T})$ is a finite set of options. A \textit{decision problem sequence} is a sequence of decision problems $\DP_1,\DP_2,...$ %
An agent for $\overline \DP$ is a sequence $\bar c$ of $c_t \in \DP_t$. 
The rewards are numbers $r_1,r_2,r_3,...\in [0,1]$. Note that in contrast to the literature on multi-armed bandit problems (\Cref{sec:rel-work}) counterfactual rewards are not defined.

It is generally helpful to imagine that (similar to multi-armed bandit problems) at each time $t$ the agent first sees $\DP_t$; then chooses $c_t$ from $\DP_t$ according to some algorithm that looks at the available options in $\DP_t$ and takes past experiences into account; then the environment calculates some reward as a function of $c_t$; the agent observes the reward and learns from it. The sequence of decision problems $\DP_t$ may in turn be calculated depending on the agent's choices. \extendedonlybit{But technically we can consider an agent who chooses $\bar c$ in the beginning without ever looking at $\overline \DP$ or $\bar r$.}

\extendedonlybit{We will often consider specific and somewhat unusual types of decision problems as examples, in particular ones where options are terms in some mathematical logic. However, our theory applies at least as well to more traditional, partly empirical decision problems. For example, one could imagine that each option describes a particular medical treatment and that the agent has to select one of the treatments for a particular patient.}

We focus on learning myopically optimal behavior. That is, we want our agent to learn to choose whatever gives the highest reward for the present decision problem, regardless of what consequences that has for future decision problems.

\section{%
Computational constraints and paradoxes of self-reference}
\label{sec:logical-uncertainty}

In this paper, we develop a normative theory of rational learning in this setting. The standard theory for rational decision making under uncertainty is
Bayesian decision theory (BDT) (\cite{Savage1954,Jeffrey1965}; for contemporary overviews, see \cite{Peterson2009,Steele2016}). %
The main ideas of this paper are motivated by a specific shortcoming of BDT: the assumption that the agent who is subject to BDT's recommendations is logically omniscient and in particular not limited by any computational constraints. \extendedonlybit{\footnote{Essentially the same issue has also sometimes been called the problem of old evidence. Cf.\ \Cref{sec:rel-work-Garrabrant}.}}
We develop a theory that gives recommendations to computationally bounded\extendedonlybit{ (and therefore in particular logically uncertain)} agents. In the following, we give two kinds of examples to illustrate the role of logical omniscience in BDT and motivate our search for an alternative theory.

\textbf{Mere intractability~~} \label{sec:betting-on-math}
The first problem is that in most realistic choice problems, it is intractable to follow BDT.\extendedonlybit{\footnote{Pointing out this type of issue with BDT has a long history in many different strands of literature, see, e.g., the overviews given by Wheeler (\citeyear[Sect.~1.3]{Wheeler2020}) and Garrabrant et al.~(\citeyear[Sect.~1.2]{Garrabrant2016}).}}
Bayesian updating and Bayes-optimal decision making are only feasible if the environment is small or highly structured\extendedonlybit{ (\cite[Sections 2.5, 5.5]{Savage1954}; \cite{Cooper1990}; \cite{Chatterjee2016})}. %
Even if the agent had a perfectly accurate world model, determining the optimal choice may require solving computationally hard problems, %
such as the traveling salesman problem, \extendedonlybit{protein design \cite{Pierce2002}, }planning in 2-player competitive games\extendedonlybit{ (e.g., \cite{Even1976}; \cite{Schaefer1978})}, etc. Optimal choice may also rely on whether particular mathematical claims are true, e.g., when assessing the safety of particular cryptographic methods. In all these problems, BDT requires the agent to perfectly solve the problem at hand. However, we would like a theory of rational choice that makes recommendations to realistic, bounded agents who can only solve such problems approximately.\extendedonlybit{\footnote{Here is an analogy to explain this critique of BDT. A trivial theory of rational choice is the following: Take the option that is best given the model that accurately describes the world. This theory assumes omniscience about both logical and empirical facts. From a BDT perspective, it is unsatisfactory because we do not know what model describes the actual world. A proponent of the trivial theory could argue that one should simply \textit{approximate} the trivial theory. But it is unclear how one should perform this approximation and from a BDT perspective this is an essential question that a theory of rational choice should answer in a principled way. Similarly, we believe that there should be a principled normative theory of how one should address problems that are hard to solve exactly.}} 

Consider a decision problem $\DP=\{a_1,a_2  \}$, where the agent knows that option $a_1$ pays off the value of the $10^{100}$-th digit of the binary representation of $\pi$. Option $a_2$ pays off $0.6$ with certainty. In our formalism, $r$ equals the $10^{100}$-th digit of the binary representation of $\pi$ if $c=a_1$ and $r=0.6$ if $c=0.6$. All that Bayesian decision theory has to say about this problem is that one should calculate the $10^{100}$-th digit of $\pi$; if it is $1$, choose $a_1$; otherwise choose $a_2$. Unfortunately, calculating the $10^{100}$-th digit of $\pi$ is likely intractable.\footnote{\label{ftn:digits-of-pi-unknown}Remote digits of $\pi$ are a canonical example in the literature on bounded rationality and logical uncertainty \citep[see][for an early usage]{Savage1967}.  
To the knowledge of the authors it is unknown whether the $n$-th digit of $\pi$ can be guessed better than random in less than $O(n)$ time. \extendedonlybit{It is (to our knowledge) not even known whether all digits appear with equal frequency in the decimal representation of $\pi$. }For a general, statistical discussion of the randomness of digits of $\pi$, see Marsaglia \citeyear{Marsaglia2005}.
} Hence, Bayesian decision theory does not have any recommendations for this problem for realistic reasoners. At the same time, we have the strong normative intuition that -- if digits of $\pi$ indeed cannot be predicted better than random under computational limitations -- it is rational to take $a_2$. %
We would like our theory to make sense of that intuition.

We close with a note on what we can expect from a theory about rational decision making under computational bounds. A na\"ive hope might be that such a theory could tell us how to optimally use some amount of compute (say, 10 hours on a particular computer system) to approximately solve any given problem (cf.\ our discussion in \Cref{sec:Russell-bounded-opt} of Russell et al.'s \cite{Russell1991,Russell1993,Russell1995} work on bounded optimality); or that it might tell us \textit{in practice} at what odds to bet on, say, Goldbach's conjecture with our colleagues. In this paper, we do not provide such a theory and such a theory cannot exist.\extendedonlybit{\footnote{
For example, Blum's \citeyear{Blum1967} speedup theorem states, roughly, that there is a decision problem such that for every algorithm solving that decision problem, there exists another, much faster algorithm solving that decision problem. Also, by, e.g., Rice's theorem, it is not even decidable, for a given computational problem, whether it can be solved within some given computational constraints. Also see Hutter et al.~(\citeyear[Sect.~7.1]{Hutter2005}) for some discussion, including a positive result, i.e., an algorithm that is in some sense optimal for all well-defined computational problems.}} We must settle for a more modest goal.
Since our agents face decision problems repeatedly, our rationality requirement will be that the agent \textit{learns} to approximately solve these problems optimally in the limit. For example, if digits of $\pi$ are pseudo-random in the relevant sense, then a rational agent must converge to betting 50-50 on remote binary digits of $\pi$. But it need not bet 50-50 ``out-of-the-box''.
\extendedonlybit{While our paper thus focuses on a general theoretical answer to the problem of intractability for rational agents, note that the perspective of assigning probabilities to logical claims has recently also been used to derive novel results in complexity theory \citep{Borak2019}. %
}

\textbf{Paradoxes of self-reference, strategic interactions, and counterfactuals~~}
\label{sec:log-unc-strategic-interactions}
\label{sec:motivation-self-reference}
A second problem with BDT and logical omniscience more generally is that it creates inconsistencies if the values of different available options depend on what the agent chooses. As an example, consider the following decision problem, which we will call the Simplified Adversarial Offer (SAO) \citep[after a decision problem introduced by][]{ExtractingMoneyFromCDT}. Imagine that an artificial agent chooses between two available alternatives $a_0$ and $a_1$, where $a_0$ is known to pay off $\nicefrac{1}{2}$ with certainty, and $a_1$ is known to pay off 1 if the agent's program run on this decision problem chooses $a_0$, and 0 otherwise.
Now assume that the agent chooses deterministically and optimally given a logically omniscient belief system%
. Then the agent knows the value of each of the options. This also means that it knows whether it will select $a_0$ or $a_1$.
But given this knowledge, the agent selects a different option than what the belief system predicts. This is a contradiction. Hence, there exists no agent that complies with standard BDT in this problem.
Compare the example\extendedonlybit{s} of Oesterheld and Conitzer \citeyear{ExtractingMoneyFromCDT} and Spencer \citeyear{Spencer2021}; also see Demski and Garrabrant (\citeyear[Sect.\ 2.1]{Demski2020}) for a discussion of another, subtler issue that arises from logical omniscience and introspection.

We are particularly interested in problems in which such failure modes apply. SAO is an extreme and unrealistic example, selected to be simple and illustrative. However, strategic interactions between different rational agents share the ingredients of this problem: Agent 1 is thinking about what agent 2 is choosing, thereby creating a kind of reference to agent 2 in agent 2's environment. We might even imagine that two AI players know each others' exact source code (cf.\ \citealp{Rubinstein1998}, Sect.\ 10.4; \citealp{Tennenholtz2004}; \citealp{Hoek2013}; \citealp{Barasz2014}; \citealp{Critch2016}; \citealp{RobustProgramEquilibrium}). Further, it may be in agent 2's interest to prove wrong whatever agent 1 believes about agent 2. \abridgedonlybit{For a closely related discussion of issues of bounded rationality and the foundations of game theory, see Binmore \citeyear{Binmore1987} and references therein (\citealp[cf.][Ch.\ 10]{Rubinstein1998}; \citealp[][Sect.\ 3.2]{Demski2020}).
}

\begin{extendedonlyblock}
Besides the failure of BDT in particular, the Adversarial Offer is illustrative of the challenge of developing a normative rationality criterion for such general decision problems. Many notions of optimal behavior are based on a requirement that a rational agent should not be outperformed by an alternative strategy.\extendedonlybit{\footnote{In multi-armed bandit problems, for example, one usually considers the goal of minimizing regret (see the discussion in \Cref{sec:rel-work-multi-armed-bandits}). BDT itself can also be motivated in this way, as is done in the complete class theorems (\citealp[][especially Chapters 1 and 3]{Wald1950}; \citealp[for introductions, see also][Chapters 1 and 2]{Ferguson1967}; \citealp[][Sect.~5.7]{Lehmann1998}).}} But any agent may end up in a decision problem sequence that at each time $t$ poses the problem $\mathrm{SAO}_{c_t}$. Regardless of what $c_t$ selects, it always selects the option with the lowest reward. Hence, the agent choosing according to the sequence of $c_t$ performs worse than any agent that deviates from $c_t$ at least some of the time.

An alternative perspective on this is that in our setting (as in the real world), counterfactual claims are problematic. Although one can resolve the value of all options in $\mathrm{SAO}_{c}$, it seems odd for an agent after choosing $a_0$ to believe, \enquote{Had I chosen $a_1$, I would have gotten $1$}. Arguably the \enquote{right} counterfactual statement in this case is, \enquote{Had I chosen $a_1$, I would have gotten $0$}, even though $a_1$ in fact resolves to 1. However, it is unclear how the \enquote{right} counterfactuals can be constructed in general. In the present paper, we therefore avoid the reliance on any such counterfactual claims even if some form of counterfactual is revealed or can be calculated \textit{ex post} by the agent. In practice another motivation to not rely on counterfactual claims is that when interacting with the real world, counterfactuals are not directly revealed.\extendedonlybit{\footnote{What counterfactuals are and what role they should play in rational choice is, of course, one of the most widely discussed questions in analytic philosophy. For an introduction to the literature on counterfactuals in general, see, for example, Starr (\citeyear{Starr2019}). Properly relating our views and the approach of the present paper to this vast literature could easily fill its own paper. Note that in the present context, \textit{logical} counterfactuals are particularly relevant, which seem even harder to make sense of.
For discussions of the role of counterfactuals (and the related concept of causality) in rational choice in particular, see, e.g., Eells (\citeyear{Eells1982}), Joyce (\citeyear{Joyce1999}), or Ahmed (\citeyear{Ahmed2014}). In \Cref{sec:rel-work-dt-Newcomb}, we will relate the present theory to a particular topic of that literature.}%
}

Without counterfactuals, we face a different problem: Imagine an agent choosing between $``\nicefrac{1}{3}"$ and $``\nicefrac{2}{3}"$. How can we design a rationality requirement that rules out an agent who simply always takes $``\nicefrac{1}{3}"$?
In the next section, we will give an answer to this question. Roughly, our approach is the following: We do not ever make claims about counterfactuals in a particular decision problem. However, we require that in a sequence of decision problems, a rational agent tests different hypotheses about what the optimal choice is. For example, there will be a hypothesis which claims that in this type of problem one should choose $``\nicefrac{2}{3}"$ and that doing so provides a payoff of $\nicefrac{2}{3}$. This hypothesis has to be tested by actually taking $``\nicefrac{2}{3}"$ and seeing whether the promised payoff of $\nicefrac{2}{3}$ was realized. This particular hypothesis keeps its promise and is therefore prudent to follow, unless another hypothesis (which has either proved reliable or is up for testing) promises an even higher reward.

For a closely related discussion of issues of bounded rationality\extendedonlybit{, counterfactuals} and the foundations of game theory, see \citet{Binmore1987} and references therein.
\end{extendedonlyblock}

\section{The rationality criterion}
\label{sec:the-criterion}

\extendedonlybit{In short, our approach is as follows: Agents have to not only choose actions, but also estimate in each round the reward they will receive. As part of our rationality criterion we require that these estimates are not systematically above %
what the agent actually obtains. Further, we consider rationality relative to some set of hypotheses, which in turn recommend actions and promise that some reward is achieved when following the recommendation. To satisfy computational constraints, we can restrict the set of hypotheses to only include efficiently computable ones.
Roughly, our rationality criterion then states that if a hypothesis infinitely often claims strictly higher reward than the agent estimates for its own choice, then the agent must test this hypothesis infinitely often. Testing requires taking the option recommended by the hypothesis in question. To reject a hypothesis, these tests must indicate that the hypothesis consistently over-promises.
}

\subsection{Preliminary definitions}

An \textit{estimating agent $\bar \alpha$} is a sequence of choices from the available options $\alpha^c_t\in  \DP_t$ and \textit{estimates} $\alpha^e_t\in [0,1]$. Our rationality criterion uses estimating agents. For brevity, we will say \textit{agent} instead of estimating agent throughout the rest of this paper.
For example, let $\mathrm{SAO}_{\alpha,t}$ be the Simplified Adversarial Offer for the agent at time $t$ as described in \Cref{sec:motivation-self-reference}. Then we might like an agent who learns to choose $\alpha_t^c=a_0$ (which pays $\nicefrac{1}{2}$ with certainty) and estimate $\alpha_t^e=\nicefrac{1}{2}$.

A \textit{hypothesis} $h$ has the same type signature as an estimating agent. When talking about hypotheses, we will often refer to the values of $h_t^e$ as promises and to the values of $h_t^c$ as recommendations.

Our rationality criterion will be relative to a particular set of hypotheses $\mathbb{H}$. In principle, $\mathbb{H}$ could be any set of hypotheses, e.g., all computable ones, all three-layer neural nets, all 8MB computer programs, etc. Generally, $\mathbb{H}$ should contain any hypothesis (i.e., any hypothesis about how the agent should act) that the agent is willing to consider, similar to the support of the prior in Bayesian theories of learning, or the set of experts in the literature on multi-armed bandits with expert advice.
Following Garrabrant et al.~\citeyear{Garrabrant2016}, we will often let $\mathbb{H}$ be the set of functions computable in $O(g(t))$ time, where $g$ is a non-decreasing function.
We will call these hypotheses \textit{efficiently computable (e.c.)}. Note that not all time complexity classes can be 
written as $O(g(t))$. For example, the set of functions computable in polynomial time cannot be written in such a way. This simplified set is used to keep notation simple. Our results generalize to more general
computational complexity classes.

\extendedonlybit{Restricting $\mathbb{H}$ to functions computable in $O(g(t))$ relates to our goal of developing computationally bounded agents (cf.\ \Cref{sec:computing-BRIAs}). It is not clear whether computational constraints related to $t$ are the most relevant -- usually an agent's computational power does not increase as time goes on. An alternative might be to let the computational constraints depend on some number specified by the decision problems themselves. This would require some extra notation and assumptions about the environment, however, without changing our analysis much. Another question is whether asymptotic bounds are more relevant than absolute bounds. After all even $O(1)$ contains hypotheses that cannot in practice be evaluated, which we could avoid by considering only hypotheses that take 10 seconds on a particular machine. We will nevertheless often use sets $\mathbb{H}$ defined by asymptotic bounds. This is done for the usual reason: asymptotic complexity classes afford closure properties that simplify analysis. For example, if two operations are in $O(g(t))$, then compositions of the two are also in $O(g(t))$.}

\subsection{No overestimation}

We now describe the first part of our rationality requirement, which is that the estimates should not be systematically above what the agent actually obtains. The criterion itself is straightforward, but its significance will only become clear in the context of the hypothesis coverage criterion of the next section.

\begin{definition}
For $T\in \mathbb{N}$, we call $\mathcal{L}_T(\bar\alpha,\bar r) \coloneqq  \sum_{t=1}^T \alpha_t^e - r_t$
the \textit{cumulative overestimation} of an agent $\bar\alpha$ on $\bar r$.
\end{definition}

\begin{definition}\label{def:no-overestimation}
We say that an agent $\bar\alpha$ for $\overline\DP,\bar r$ \textit{does not overestimate (on average in the limit)} if $ \mathcal{L}_T(\bar\alpha,\bar r) / T \leq 0$ as $T\rightarrow \infty$.
\end{definition}

In other words, for all $\epsilon >0$, there should be a time $t$ such that for all $T>t$, $ \mathcal{L}_T(\bar\alpha,\bar r) / T \leq \epsilon$.
Note that the per-round overestimation of boundedly rational inductive agents as defined below will usually but need not always converge to 0; it can be negative in the limit\extendedonlybit{ (see \Cref{appendix:behavior-on-subsequences})}.

\subsection{Covering hypotheses}

We come to our second requirement, which specifies how the agent $\bar \alpha$ relates to the hypotheses in $\mathbb{H}$.

\begin{definition}
We say that \textit{$\bar h$ outpromises $\bar \alpha$} or that \textit{$\bar \alpha$ rejects $\bar h$ at time $t$} if $h_t^e>\alpha_t^e$.
\end{definition}

We distinguish two kinds of hypotheses:
First, there are hypotheses that promise higher rewards than $\bar\alpha^e$ in only finitely many rounds. For example, this will be the case for hypotheses that $\bar\alpha$ trusts and takes into account when choosing and estimating. Also, this could include hypotheses who recommend an inferior option with an accurate estimate, e.g., hypotheses that recommend $``\nicefrac{1}{3}"$ and promise $\nicefrac{1}{3}$ in $\{ ``\nicefrac{1}{3}",``\nicefrac{2}{3}" \}$. For all of these hypotheses, we do not require anything of $\bar\alpha$. In particular, $\bar \alpha$ need not test these hypotheses.
Second, some hypotheses do infinitely often outpromise $\bar \alpha^e$. For these cases, we will require our boundedly rational inductive agents to have some reason to reject these hypotheses. To be able to provide such a reason, $\bar \alpha$ needs to test these hypotheses infinitely often.\extendedonlybit{\footnote{If we only test them finitely many times, a correct hypothesis may be rejected due to bad luck (e.g., if rewards are random, as discussed in \Cref{sec:true-randomness}).}} \extendedonlybit{For the reasons described in \Cref{sec:log-unc-strategic-interactions}, t}\abridgedonlybit{T}esting a hypothesis requires choosing the hypothesis' recommended action.

\begin{definition}\label{def:test-set}
We call a set $M\subseteq \mathbb{N}$ a \textit{test set} of $\bar \alpha$ for $\bar h$ if for all $t\in M$, $\alpha^c_t=h^c_t$.
\end{definition}

For $\bar\alpha$ to infinitely often reject $\bar h$, these tests must then show that $\bar h$ is not to be trusted (in those rounds in which they promise a reward that exceeds $\bar \alpha^e$). That is, on these tests, the rewards must be significantly lower than what the hypothesis promises. We thus introduce another key concept.

\begin{definition}
Let $\bar h$ be a hypothesis and $M\subseteq \mathbb{N}$ be a test set of $\bar\alpha$ for $\bar h$. We call $l_T(\bar\alpha,\bar r,M,\bar h)\coloneqq \sum_{t\in M_{\leq T}} r_t - h^e_t$ \textit{the (empirical) record of $h$ (on $M$)}.
\end{definition}

Here, $M_{\leq T}\coloneqq \{ t\in M\mid t\leq T \}$ is defined to be the set of elements of $M$ that are at most $T$.

We now have all the pieces together to state the coverage criterion, which specifies how we want our agents to relate to the hypotheses under consideration.

\begin{definition}
Let $\bar\alpha$ be an agent, $\bar h$ be a hypothesis, and let $B$ be the set of times $t$ at which $\bar \alpha$ rejects $\bar h$. We say that \textit{$\bar \alpha$ covers $\bar h$ with test set $M$} if either $B$ is finite or the sequence $\left( l_T(\bar \alpha,\bar r,M,\bar h) \right)_{T\in B}$ goes to negative infinity.
\end{definition}

\subsection{The boundedly rational inductive agent criterion} We now state the BRIA criterion, the main contribution of this paper.

\begin{definition}\label{def:BRIA}
Let $\bar\alpha$ be an agent for $\overline \DP,\bar r$. Let $\mathbb{H}=\{h_1,h_2,...\}$ be a set of hypotheses.
We say $\bar\alpha$ is a \textit{boundedly rational inductive agent (BRIA) for $\overline \DP,\bar r$ covering $\mathbb{H}$ with test sets $M_1,M_2,...$} if $\bar\alpha$ does not overestimate and for all $i$, $\bar\alpha$ covers $h_i$ with test set $M_i$.
\end{definition}

In the following, whenever $\bar \alpha$ is a BRIA, we will imagine that the test sets are given as a part of $\bar \alpha$. For example, if we say that $\bar \alpha$ is computable in, say, time polynomial in $t$, then we will take this to mean that $\bar \alpha$ together with a list at time $t$ of tested hypotheses can be computed in polynomial time.

\subsection{Examples}

\textbf{Betting on digits of $\pi$~~}
Consider the decision problem sequence with $\mathrm{DP}_t=\{a_t^{\pi},x_t \}$ for all $t$, where $a_t^{\pi}$ pays off the $2^t$-th binary digit of $\pi$ -- i.e., $r_t$ is the $2^t$-th digit of $\pi$ if $\alpha_t^c=a_t^{\pi}$ -- and $x_t\in [0,1]$ pays off $x_t$. As usual we assume that the $2^t$-th binary digits of $\pi$ are pseudorandom (in a way we will make precise in \Cref{sec:pseudo-lotteries}) uniformly distributed (as they seem to be, cf.\ \cref{ftn:digits-of-pi-unknown}).
We would then expect boundedly rational agents to (learn to) choose $a_t^{\pi}$ when $x_t<\nicefrac{1}{2}$ and choose $x_t$ when $x_t>\nicefrac{1}{2}$.

We now consider an agent $\bar\alpha$ for this decision problem sequence. We will step-by-step impose the components of the BRIA criterion on $\bar\alpha$ to demonstrate their meaning and (joint) function in this example. We start by imposing the no overestimation criterion on $\bar\alpha$ without any assumptions about hypothesis coverage -- what can we say about $\bar\alpha$ if we assume that does not overestimate? As noted earlier, the no overestimation criterion alone is weak and in particular does not constrain choice at all. For instance, $\bar\alpha$ might always choose $\alpha^c_t=a_t^{\pi}$ and alternate estimates of $0$ and $1$; or it might always choose $x_t$ and estimate $x_{t-1}$.

We now impose instances of the hypothesis coverage criterion. We start with the hypothesis $ h_x$ which always recommends choosing $x_t$ and promises a reward of $x_t$. Note that for all we know about the decision problem sequence this hypothesis does not give particularly good recommendations. However, in the context of our theory, $h_x$ is useful because it always holds its promises. In particular, $h_x$'s empirical record on any test set is $0$. Hence, if $\alpha$ is to cover $h_x$, then $\alpha$ can only reject $h_x$ finitely many times. By definition, this means that $\alpha_t^e\geq x_t$ for all but finitely many $t\in \mathbb{N}$.
With the no overestimation criterion, it follows that $\alpha$ on average obtains utilities at least equal to $x_t$. But $\alpha$'s choices may still not match our bounded ideal. For example, $\alpha$ may always choose $x_t$.

Next, consider for $\epsilon>0$, the hypothesis $h_{\pi}^{\epsilon}$ that always recommends $a_t^{\pi}$ and estimates $\nicefrac{1}{2}-\epsilon$. Whether $h_{\pi}^{\epsilon}$ holds its promises is a more complicated question. But let us assume that $\bar \alpha$ covers $h_{\pi}^{\epsilon}$ with some test set $M$, and let us further assume that whether $t\in M$ is uncorrelated with the $2^t$-th binary digit of $\pi$, for instance, because predicting the $2^t$-th binary digit of $\pi$ better than random cannot be done using the agent's computational capabilities. Then $h_{\pi}^{\epsilon}$'s empirical record on $M$ will go to $\infty$, assuming that $M$ is infinite -- after all, following $h_{\pi}^{\epsilon}$'s recommendations yields a reward of $\nicefrac{1}{2}$ on average, exceeding its promises of $\nicefrac{1}{2}-\epsilon$.\extendedonlybit{ (Note that if the $2^t$-th binary digits of $\pi$ act like random variables, then this would presumably not be true for $\epsilon=0$, due to the well-known recurrence (a.k.a.\ Gambler's ruin) result about the simple symmetric random walk on the line \cite{Polya1921}.)} With the assumption that $\bar\alpha$ covers $h_{\pi}^{\epsilon}$, it follows that for all but finitely many $t$, $\alpha^e_t\geq \nicefrac{1}{2}-\epsilon$. Now imagine that $\alpha$ not only covers one particular $h_{\pi}^{\epsilon}$, but that there exist arbitrarily small positive $\epsilon$ such that $\alpha$ covers the hypothesis $h_{\pi}^{\epsilon}$. Then it follows that in the limit as $t\rightarrow\infty$, $\alpha_t^e\geq \nicefrac{1}{2}$.

The above three conditions -- no overestimation, coverage of $h_x$ and coverage of $h_{\pi}^{\epsilon}$ for arbitrarily small $\epsilon$ -- jointly imply that $\bar\alpha$ exhibits the desired behavior. Specifically, we have shown that $\bar\alpha$ must estimate at least $\max\{\nicefrac{1}{2},x_t\}$ in the limit. By the no overestimation criterion, $\bar\alpha$ also has to actually obtain at least $\max\{\nicefrac{1}{2},x_t\}$ on average. And if $\bar\alpha$ cannot guess the $2^t$-th digits of $\pi$ better than random, then the only way to achieve $\max\{\nicefrac{1}{2},x_t\}$ on average is to follow with limit frequency $1$ the policy of choosing $a_t^{\pi}$ when $x_t<\nicefrac{1}{2}$ and $x_t$ when $x_t>\nicefrac{1}{2}$.

\textbf{Adversarial offers~~}
Let $\alpha$ be an agent who faces a sequence of instances of SAO. In particular at time $t$, the agent faces $\mathrm{SAO}_{\alpha,t}=\{ a_0,a_1\}$, where $a_0$ pays off $\nicefrac{1}{2}$ with certainty. Intuitively, $a_1$ is evaluated to $1$ if on the present problem $\alpha$ chooses $a_0$ and to $0$ otherwise. Note, however, that the former fact is never relevant to computing $r_t$. So effectively $r_t=\nicefrac{1}{2}$ if $\alpha_t^c=a_0$ and $r_t=0$ otherwise.

Assume that $\alpha$ does not overestimate and that it covers the hypothesis $h$ which estimates $\nicefrac{1}{2}$ and recommends $a_0$ in every round. Hypothesis $h$ will always have an empirical record of $0$ on any test set $M$ since it holds its promises exactly. Hence, if $\alpha$ is to cover $h$, it can reject $h$ only finitely many times. Thus, $\alpha_t^e\geq \nicefrac{1}{2}$ in all but finitely many rounds. To satisfy the no overestimation criterion, $\alpha$ must therefore obtain rewards of at least $\nicefrac{1}{2}$ on average in the limit. Since $a_1$ pays off $0$ whenever it is taken by $\alpha$, it must be $\alpha_t^c=a_0$ with limit frequency $1$.

\section{Computing boundedly rational inductive agents}
\label{sec:computing-BRIAs}

As described in \Cref{sec:logical-uncertainty}, the goal of this paper is to formulate a rationality requirement that is not self-contradictory and that can be satisfied by computationally bounded agents. Therefore, we must show that one can actually construct BRIAs for given $\mathbb{H}$ and that under some assumptions about $\mathbb{H}$, such BRIAs are computable (within some asymptotic bounds).

\begin{restatable}{theorem}{computingBRIAsthm} \label{thm:computable-BRIAs}
Let $\mathbb{H}$ be a computably enumerable set consisting of ($O(g(t))$-)computable hypotheses. (Let $g\in \Omega(\log)$.) Then there exists an algorithm that computes a BRIA covering $\mathbb{H}$ (in $O(g(t)q(t))$, for arbitrarily slow-growing, $O(g(t))$-computable $q$ with $q(t)\rightarrow\infty$) for any $\overline \DP, \bar r$.
\end{restatable}

We here give a sketch of our construction. For each decision problem, we run a first-price sealed-bid auction among the hypotheses. The highest-bidding hypothesis determines the agent's choice and estimate and is tested in this round. For each hypothesis, we maintain a wealth variable that tracks the hypothesis' empirical record. A hypothesis' bid is bound by its wealth. Thus, when a hypothesis outpromises the agent, this implies that the hypothesis' wealth is low. Upon winning an auction, the hypothesis pays its promise and gains the reward obtained after following the hypothesis' recommendation. We further distribute at each time $t$ allowance to the hypotheses. The overall allowance per round is finite and goes to zero. The cumulative allowance for each hypothesis goes to $\infty$ over time. Thus, if a hypothesis is rejected infinitely often, then this requires the hypothesis to have spent all allowance and thus for its record among those rejection rounds to go to $-\infty$. Moreover, the cumulative overestimation is bound by overall allowance distributed and thus per-round overestimation goes to $0$.

\extendedonlybit{In \Cref{appendix:computing-BRIAs}, we provide a construction for BRIAs and prove that it has the claimed computability properties. It can similarly be shown that, for example, a BRIA relative to the class $P$ of hypotheses computable in polynomial time can be computed in arbitrarily close to polynomial time, i.e.\ in $O(t^{q(t)})$  for arbitrarily slow-growing $q$ with $q(t)\rightarrow\infty$.}

The next result shows that the BRIAs given by \Cref{thm:computable-BRIAs} are optimal in terms of complexity.

\begin{restatable}{theorem}{noecBRIA}\label{thm:no-ec-BRIA}
Let $\alpha$ be a BRIA for $\overline \DP,\bar r,\mathbb{H}$. Assume that there are infinitely many $t$ such that $|\DP_t|\geq 2$ and $\alpha_t^e<1$. If $\mathbb{H}$ is the set of ($O(g(t))$-)computable hypotheses, then $\alpha$ is not computable (in $O(g(t))$).
\end{restatable}

\extendedonlybit{We prove (and discuss) this in \Cref{sec:no-ec-BRIA}.}

\section{Lower bounds on average rewards}
\label{sec:lower-bounds-on-avg-reward}

\abridgedonlybit{\textbf{Options with payoff guarantees~~}}
\label{sec:easy-options}
Throughout this section, we will show that BRIAs satisfy many desiderata that one might have for rational decision makers. We start with a simple result which shows that if at each time $t$ one of the options can be efficiently shown to have a value of at least $L_t$, then a BRIA will come to obtain at least $L_t$ on average.

\begin{restatable}{theorem}{easyoptionsthm}\label{thm:lower-bounds-from-easy-options}
Let $\bar \alpha$ be a BRIA for $\overline \DP,\bar r$ and the set of e.c.\ hypotheses. Let $\bar a$ be a sequence of terms in $\mathcal{T}$ s.t.\ for all $t\in \mathbb{N}$, it holds that $a_t\in \DP_t$ and $
\alpha_t^c= a_t \implies r_t \geq L_t$
for some e.c.\ sequence $\bar L$. We require also that the $a_t$ are efficiently identifiable from the sets $\DP_t$. Then in the limit as $T\rightarrow \infty$ it holds that $\sum_{t=1}^T r_t / T\geq \sum_{t=1}^T L_t / T$.
\end{restatable}

\extendedonlybit{
}

\extendedonlybit{A formal proof is given in \Cref{appendix:proof-of-easy-options-thm}.} The proof idea is simple. Consider the hypothesis that estimates $L_t$ and recommends $a_t$ if $t\in S$ and promises $0$ otherwise. This hypothesis always keeps its promises. Hence, to cover this hypothesis, $\alpha$ can be outpromised by this hypothesis only finitely many times.

\extendedonlybit{\Cref{thm:lower-bounds-from-easy-options} implies that when the value of all options is e.c., then a BRIA must choose the best available option. For example, when the choice is between $``\nicefrac{1}{3}"$ and $``\nicefrac{2}{3}"$, a BRIA has to choose $``\nicefrac{2}{3}"$ with frequency 1.}

We can interpret \Cref{thm:lower-bounds-from-easy-options} as providing an immunity to money extraction schemes, a widely discussed rationality condition. If a BRIA can leave with a certain payoff of $L_t$, it will on average leave with at least $L_t$. For example, in SAO of \Cref{sec:log-unc-strategic-interactions}, a BRIA walks away with at least $\nicefrac{1}{2}$, which in turn means that it chooses $a_0=``\nicefrac{1}{2}"$ with frequency 1. \extendedonlybit{As \citet{ExtractingMoneyFromCDT} and \citet{Spencer2021} show, a different normative theory of rationality, called causal decision theory, can be used as a money pump with this example.}%

\extendedonlybit{Another corollary of \Cref{thm:lower-bounds-from-easy-options} is that BRIAs must learn and use empirical facts that can be efficiently deduced from what is revealed by $\overwrite \DP$ and $\bar r$. For example, imagine that in one round, $\DP_t,r_t$ reveals that the minimum of the populations of Hamburg and Amsterdam is 0.8 million. Then in later rounds, this information can be used to efficiently compute lower bounds on other options. For example, the option that pays off the maximum of the populations of Hamburg and Detroit in millions can be deduced to be at least 0.8. If such decision problems occur infinitely often, BRIAs must converge to exploiting such inferences.}

\textbf{Options with algorithmically random payoffs~~}
\label{sec:pseudo-lotteries}
\Cref{thm:lower-bounds-from-random-variables} only tells us something about \textit{truly} random variables. But a key goal of our theory is to also be able to assign expected rewards to %
\textit{algorithmically} random sequences, i.e., sequences that are deterministic\extendedonlybit{ and potentially even computable}, but relevantly unpredictable under computational constraints. We first offer a formal notion of algorithmic randomness.

\begin{restatable}{definition}{vMWCdef}\label{def:bounded-vMWC-randomness}
We say a sequence $\bar y$ is \textit{($O(h(t))$ boundedly) van Mises--Wald--Church (vMWC) random with means $\bar\mu$} if for every infinite set $S\subseteq \mathbb{N}$ that is decidable (in $O(h(t))$ time) from available information%
, we have that $\lim_{T\rightarrow \infty}\sum_{t\in S_{\leq T}} y_t-\mu_t=0$.
\end{restatable}

Thus, we call a sequence random if there is no ($O(g(t))$-)computable way of selecting in advance members of the sequence whose average differs from the means $\bar \mu$.
\Cref{def:bounded-vMWC-randomness} generalizes the standard definition of (unbounded) vMWC randomness \citep[e.g.][Definition 7.4.1]{Downey2010} to non-binary values with means $\bar\mu$ other than $\nicefrac{1}{2}$ and computational constraints with outside input (e.g., from $\overline \DP$, which could contain options containing information such as, \enquote{by the way, the trillionth digit of $\pi$ is 2}). The notion of vMWC randomness is generally considered quite weak (\citealp[e.g.][Sect.\ 6.2]{Downey2010}).

\begin{restatable}{theorem}{pseudolotteriesthm}\label{thm:pseudo-lotteries}
Let $\bar\mu$ be an e.c.\ sequence on $[0,1]$. Let $\alpha$ be an $O(h(t))$-computable BRIA for decision problem sequence $\overline \DP$ with rewards $\bar r$ covering all e.c.\ hypotheses. Let $\bar a$ be a sequence of terms in $\mathcal{T}$ s.t.\ $a_t\in \DP_t$ for all $t\in \mathbb{N}$ and the payoffs $r_t$ in rounds with $\alpha_t^c=a_t$ are $O(h(t))$-boundedly vMWC random with means $\bar \mu$. Then in the limit as $T\rightarrow\infty$, it holds that $\sum_{t=1}^T r_t/T \geq \sum_{t=1}^T \mu_t/T$.
\end{restatable}

We show an analogous result for Schnorr bounded randomness \cite{Schnorr1971,AmbosSpies1997,Wang2000,Stull2020} in \Cref{appendix:schnorr}. %
\extendedonlybit{Note that we could replace the $=$ sign in the last line of the definition with a $\geq$ and all of the following would still hold -- however, the resulting definition does not reasonably capture randomness.}
Analogous results for truly random options follow from results for algorithmically random $\bar r$ and the fact that a sequence of truly random, independent numbers is algorithmically random almost surely. We give a direct proof in
\Cref{appendix:lotteries-thm}.

\section{Boundedly rational inductive agents as a foundation for game theory}
\label{sec:game-theory}

\extendedonlybit{
\subsection{Games as decision problems}
}
\label{sec:game-setup}

We first recap basic game-theoretic concepts.\extendedonlybit{ For a thorough introduction to game theory, see Osborne \citeyear{Osborne2004} or any other textbook on the topic.} A \textit{(two-player) game} consists of two finite sets of \textit{(pure) strategies} $A_1,A_2$, one set for each player, and two payoff functions $u_1,u_2{\colon} A_1{\times} A_2 {\rightarrow} [0,1]$. %
A correlated strategy profile is a distribution $\mathbf{c}\in \Delta(A_1\times A_2)$ over $A_1\times A_2$. We can naturally extend utility functions to correlated strategy profiles as follows: $u_i(\mathbf{c})=\sum_{\mathbf{a}\in A_1\times A_2} c_{\mathbf{a}} u_i(\mathbf{a})$.
We call a correlated strategy profile $\mathbf{c}$ \textit{strictly individually rational} if each player's payoff in $\mathbf{c}$ is greater than their pure strategy maximin payoff, i.e., $u_i(\mathbf c) > \max_{a_i\in A_i}\min_{a_{-i}\in A_{-i}} u_i(a_i,a_{-i})$.

\begin{extendedonlyblock} %
\begin{table}
	\begin{center}
    \setlength{\extrarowheight}{2pt}
    \begin{tabular}{cc|C{2cm}|C{2cm}|}
      & \multicolumn{1}{c}{} & \multicolumn{2}{c}{Player $2$}\\
      & \multicolumn{1}{c}{} & \multicolumn{1}{c}{Cooperate}  & \multicolumn{1}{c}{Defect} \\\cline{3-4}
      \multirow{2}*{Player $1$}  & Cooperate & $0.3,0.3$ & $0.1,0.4$ \\\cline{3-4}
      & Defect & $0.4,0.1$ & $0.2,0.2$ \\\cline{3-4}
    \end{tabular}
    \end{center}
    \caption{A payoff matrix for the Prisoner's Dilemma}
    	\label{table:PD-payoffs}
  \end{table}
\end{extendedonlyblock}

Now imagine that two BRIAs $\bar\alpha_1,\bar\alpha_2$ learn to play a game against each other. That is, we consider BRIAs $\bar{\alpha}_1,\bar{\alpha}_2$ for $\bar\DP^{\bar\alpha_1},\bar\DP^{\bar\alpha_2}$ respectively, where $\bar\DP^{\bar\alpha_i}=A_i$ for $i=1,2$ and $r_{i,t}=u_i(\alpha_{1,t}^c,\alpha_{1,t}^c)$.
\extendedonlybit{Abusing notation a little, we use $a_i\in A_i$ to represent the available options in $\mathrm{DP}^{\alpha_i}_t$. For instance, we write $\alpha_{i,t}^c=a_1$ to denote that $\alpha_i$ chooses the option from $\mathrm{DP}^{\alpha_i}_t$ that corresponds to $a_i\in A_i$.}

\begin{extendedonlyblock}
Note that this is a fairly specific setup. Other versions are possible. For example, instead of knowing the opponent's source code or mathematical definition precisely, we could imagine that they have some distribution over opponent BRIAs. After all, if we accept our BRIA criterion as a definition of rationality, then the common rationality assumption underlying game theory still leaves open which exact BRIA the other player uses.
\end{extendedonlyblock}

\begin{restatable}[Folk theorem]{theorem}{FolkTheorem}\label{thm:folk-theorem}
Let $\Gamma$ be a game. Let $\mathbb{H}_1,\mathbb{H}_2$ be any sets of hypotheses. Let $\mathbf{c}\in \Delta(A_1\times A_2)$ be strictly individually rational. Then there exists $\mathbf{c}'$ arbitrarily close to $\mathbf{c}$ and BRIAs $\bar\alpha_1,\bar\alpha_2$ covering $\mathbb{H}_1,\mathbb{H}_2$ for decision problem sequences $\overline \DP^{\alpha_1},\overline \DP^{\alpha_2}$ with rewards $\bar r_1,\bar r_2$ based on $\Gamma$ as defined above s.t.~the empirical distribution of $(\alpha_1^c,\alpha_2^c)$ converges to $\mathbf{c}'$, i.e., for all $\mathbf{a}\in A_1\times A_2$, $\nicefrac{1}{T}\sum_{t=1}^T \mathbbm{1}[(\alpha_1^c,\alpha_2^c)=\mathbf{a}] \rightarrow c'_{\mathbf{a}}$ as $T\rightarrow \infty$.
Conversely, if $\alpha_1,\alpha_2$ are BRIAs for sets of hypotheses $\mathbb{H}_1$ and $\mathbb{H}_2$ that contain at least the constant-time deterministic hypotheses, $\sum_{t=1}^T u_i(\alpha_{1,t}^c,\alpha_{2,t}^c)/T \geq \max_{a_i} \min_{a_{-i}} u_i(a_i,a_{-i}) \text{ as }T\rightarrow \infty$.   
That is, in the limit each player receives at least their maximin utility.
\end{restatable}

\Cref{thm:folk-theorem} is compelling, because it means BRIAs can learn to cooperate in one-shot games where rational agents would otherwise fail to cooperate (e.g., contrast fictitious play, or regret learning, both of which necessarily converge to defecting in the Prisoner's Dilemma). 
Note that our BRIA criterion is myopic, i.e., aimed at maximizing reward in the \textit{current} round. Thus, even though the BRIAs in the above setting play repeatedly, the above result is unrelated to the folk theorems for repeated games.

\extendedonlybit{The proof is given in \Cref{appendix:proof-of-folk-theorem}.}

\section{Related work}
\label{sec:rel-work}

\textbf{Multi-armed bandit problems~~}
\label{sec:rel-work-contextual-bandits}
Our setting resembles a multi-armed bandit problem with expert advice (where $\mathbb{H}$ is the set of ``experts''). The main difference is that we only define $r_t$, the reward actually obtained by the agent. The literature on multi-armed bandit problems assumes that the problem also defines the (counterfactual) rewards of untaken options and defines rationality in terms of these rewards. As discussed in \Cref{sec:motivation-self-reference}, one of our motivations is to do away with these counterfactuals.

\begin{extendedonlyblock}
 Within the literature on multi-armed bandit problems, some strands of work in statistical learning theory make assumptions that avoid the problems of bounded rationality and paradoxes of self reference. 
For example, Yang and Zhu~(\citeyear[Assumption A in Sect.~5]{Yang2002}) and Agarwal et al.~(\citeyear[Assumption 1 in Sect.~2]{Agarwal2012}) assume that the agent can converge to having a fully accurate model of how the available actions give rise to rewards.
\extendedonlybit{Other papers explicitly assume that the reward is determined by some linear function \citep[see][Chapters 19--29 for an overview]{Lattimore2020}.} These assumptions allow a much simpler rationality requirement, namely some kind of convergence to optimal behavior (cf.\ \Cref{sec:easy-options}). \extendedonlybit{Aside from the early (and very general) work of \citet{Yang2002}, the literature on contextual multi-armed bandits has come to focus on achieving fast convergence rates (which we have given little consideration in this paper).}
\end{extendedonlyblock}

\begin{extendedonlyblock}
As we have argued in \Cref{sec:logical-uncertainty}, computationally complex reward functions pose quite different theoretical problems, including the impossibility of deciding based on accurate beliefs about the available options and of low-regret learning. We have argued that facing these issues head-on is important, e.g., for studying strategic interactions. We suspect that authors in this line of work generally do not have such problems in mind and are instead inspired by settings in which uncertainty is primarily empirical and computationally simple models can be somewhat accurate, e.g., when selecting treatments for a patient based on medical data.
\end{extendedonlyblock}

\label{sec:bandits-with-expert-advice}

Within the multi-armed bandit literature, the most closely related strand of work is the literature on adversarial multi-armed bandit problems with expert advice (\citealp[][Sect.~7]{Auer2001}; \citealp[][Chapter 18]{Lattimore2020}). %
Like this paper, this literature addresses this problem of bounded rationality by formulating rationality relative to a set of hypotheses (the eponymous experts). However, its rationality criterion is very different from ours: they require regret minimization and in particular that cumulative regret is sublinear, a condition sometimes called Hannan-consistency. As the Simplified Adversarial Offer shows, Hannan-consistency is not achievable in our setting. However, it does become achievable if we assume that the agent has access to a source of random noise that is independent from $\overline \DP$ \citep[see, e.g, the Exp4 algorithm of][Sect.~7]{Auer2001}. Importantly, the rationality criterion itself ignores the ability to randomize, i.e., it does not prescribe that the use of randomization be optimal in any sense.

We find it implausible to \textit{require} rational agents to randomize to minimize regret; most importantly, regret minimization can require minimizing the rewards one actually obtains -- see \Cref{appendix:on-randomization-and-regret}.%
\extendedonlybit{At the same time, we conjecture that learners with low regret relative to a set of hypotheses $\mathbb{H}$ satisfy a version of the BRIA criterion; see \Cref{appendix:regret-minimizers-satisfy-BRIA} for a preliminary result.}

\begin{extendedonlyblock}
One interesting issue in the literate on multi-armed bandit problems with expert advice is that of reactive (a.k.a.\ non-oblivious) bandits. For example, there could be a bandit/decision problem sequence that at each time $t>T$ pays the agent a dividend if the agent invests (i.e., foregoes a small reward) on day $T$. Like the BRIA criterion of this paper, standard notions of Hannan-consistency are myopic and therefore require that one learns to take the small reward today. Depending on the setting, this may be undesirable. Some authors have therefore explicitly considered the goal of maximizing reward non-myopically (\citealp{Farias2006}; \citealp[][Sect.~7.11]{Cesa-Bianchi2006}; \citealp{Arora2012}). However, as these authors have noted, it is in general difficult to define sensible non-myopic notions of regret. The underlying problem is essentially the problem of making counterfactual claims that motivates much of the present paper (see \Cref{sec:log-unc-strategic-interactions}). Only one trajectory is observed and in general it is difficult to evaluate claims about \enquote{what would have happened} if alternative strategies has been used. In the theory of multi-armed bandits, this problem is usually addressed by making assumptions that ensure that variants of typical notions of regret can be applied after all. In particular, it is assumed that the bandit is forgetful. Since BRIA theory does not rely on counterfactual claims, we believe that BRIA theory can be used to address this problem more generally and satisfactorily. It seems that one merely has to adapt the BRIA theory to incorporate non-myopia. This can be done, for example, by evaluating a bid $h_{i,t}$ not based on the immediate reward $r_t$ obtained after accepting it ($h_{i,t}^c=\alpha_t^c$) but on the discounted reward
\begin{equation}
\sum_{t'=t}^\infty \gamma^{t'-t} r_t.
\end{equation}
\end{extendedonlyblock}

\extendedonlybit{\subsection{Decision theory of Newcomb-like problems}}
\abridgedonlybit{\textbf{Decision theory of Newcomb-like problems}~~}
\label{sec:rel-work-dt-Newcomb}
Problems in which the environment explicitly predicts the agent have been discussed as Newcomb-like problems by (philosophical) decision theorists \citep{Nozick1969}\extendedonlybit{\cite{Ahmed2014}}.\extendedonlybit{ In fact, the Adversarial Offer of \citet{ExtractingMoneyFromCDT} is intended as a contribution to that theory.}
Most of this literature has focused on discussing relatively simple cases (similar to SAO)\extendedonlybit{ in which people have strong but differing intuitions about what the rational choice should be}. In these cases, BRIAs generally side with what has been called evidential decision theory. For example, by \Cref{thm:lower-bounds-from-easy-options}, BRIAs learn to one-box in Newcomb's problem\extendedonlybit{, cooperate in a Prisoner's Dilemma against an exact copy and choose $a_0$ in the Adversarial Offer}.\extendedonlybit{\footnote{The underlying reason is roughly that BRIAs implement a version of what has been called the law of effect (\citealp{Thorndike1911}; \citealp[cf.][Sect.~1.6]{Sutton1998}), which roughly states that behaviors will be repeated if they have been followed by high rewards in the past. As has been pointed out by, e.g., \citet{Gardner1973} and \citet{LearningDT}, learning according to the law of effect yields evidential decision theory-like behavior.

Of course, many evidential decision theorists may disagree with particular recommendations that the present theory makes. For example, while BRIAs learn to cooperate against exact copies of themselves, a pair of sufficiently different BRIAs will learn to defect against each other (see \Cref{sec:Nash-equilibrium}). In contrast, some have argued or simply assumed that evidential decision theory-type reasoning should lead to cooperation more generally \citep[e.g.][]{Hofstadter1983}.}} Of course, BRIAs differ structurally from how a decision theorist would usually conceive of an evidential decision theory-based agent. E.g., BRIAs are not based on expected utility maximization (though they implement it when feasible; see \Cref{sec:true-randomness}). We also note that the decision theory literature has, to our knowledge, not produced any formal account of how to assign the required conditional probabilities in Newcomb-like problems.

\extendedonlybit{\subsection{Bounded rationality}}
\abridgedonlybit{\textbf{Bounded rationality}~~}
The motivations of the present work as per \Cref{sec:logical-uncertainty}, especially \Cref{sec:betting-on-math}, coincide with some of the motivations for the study of bounded rationality. \extendedonlybit{For instance, in one of his seminal works, \citet[p.\ 99]{Simon1955} writes that \enquote{the task is to replace the global rationality of economic man with a kind of rational behavior that is compatible with [...] the computational capacities that are actually possessed by organisms, including man}. Compare \citet[Sect.\ 1.3]{Wheeler2020}.} However, other motivations have been given for the study of bounded rationality as well \citep[see, e.g.,][Sect.\ 2]{Selten1990}. More importantly, since much of bounded rationality is geared towards explaining or prescribing \textit{human} (as opposed to AI) behavior, the characterization and analysis of \enquote{computational capacities} often differ from ours \citep[e.g.][]{Conlisk1996}. For instance, for most humans dividing 1 by 17 is a challenge, while such calculation are trivial for computers. \extendedonlybit{(Meanwhile, the brain performs many operations (e.g., recognizing objects in images) that require much more complex computations.)}
A few authors have also explicitly connected the general motivations of bounded rationality with paradoxes of self reference and game theory as discussed in \Cref{sec:log-unc-strategic-interactions} (\cite{Binmore1987}, \cite[Ch.\ 10]{Rubinstein1998}).
Anyway, the literature on bounded rationality is vast and diverse. Much of it is so different from the present work that a comparison hardly makes sense. Below we discuss a few approaches in this literature that somewhat resemble ours. In particular, like the present paper (and Hannan consistency) they specify rationality relative to a given set of hypotheses (that in turn is defined by computational constraints).

\extendedonlybit{\subsection{Russell et al.'s bounded optimality}}
\abridgedonlybit{\textbf{Russell et al.'s bounded optimality~~}}
\label{sec:Russell-bounded-opt}
Like our approach and the other approaches discussed in this related work section, Russell et al.\ define \textit{bounded optimality} as a criterion relative to a set of (computationally bounded) hypotheses called \textit{agent programs} (\citealp{Russell1991}, Sect.\ 1.4; \citealp{Russell1993}; \citealp{Russell1995}). Roughly, an agent program is boundedly optimal if it is the optimal program from some set of bounded programs.\extendedonlybit{ For example, imagine that the agent will only face a single decision problem of type $\{ a_m,x \}$, where $a_m$ is known to pay off the $m$-th binary digit of $\pi$ and $x\in [0,1]$ simply pays $x$. Then Russell et al.\ have us ask questions such as: among all computer programs of size at most $16$MB that return an answer (i.e., either $a_m$ or $a_x$) after running for at most a day on a specific computer, what program maximizes expected reward if $m$ is sampled from $\mathbb{N}$ with probability proportional to $\nicefrac{1}{m^2}$ and $x$ is sampled uniformly from $[0,1]$? For large enough $m$, the optimal bounded program would simply choose $a_m$ whenever $x<\nicefrac{1}{2}$, and $a_x$ whenever $x>\nicefrac{1}{2}$, in line with our approach.}
The main difference between our and Russell et al.'s approach is that we address the problems of \Cref{sec:logical-uncertainty} by developing a theory of learning to make such decisions, while Russell et al.\ address them by moving the decision problem one level up, from the agent to the design of the agent \citep[cf.][Sect.\ 2.2 for a discussion of this move]{Demski2020}. As one consequence, we can design general BRIAs, while it is in general hard to design boundedly optimal agents. \extendedonlybit{(\citet[][Sect.\ 4]{Russell1995} give a special class of environments in which they show the design of boundedly optimal agents to be tractable.)}
Of course, the feasibility of designing BRIAs comes at the cost of our agents only behaving reasonably in the limit. \extendedonlybit{As an example, imagine that the agent will with probability $1$ be offered some bet on Goldbach's conjecture. Then Russell et al.'s approach requires the agent's designer to determine whether Goldbach's conjecture is true. In contrast, our approach puts no requirement on an agent who only faces this one betting situation.} Moreover, the designer of boundedly optimal agents as per Russell et al.\ may become a subject of the paradoxes of \Cref{sec:log-unc-strategic-interactions} in problematic ways. \extendedonlybit{Imagine that the designer is in turn some computer program $d$ and let us say that $d$ is posed the problem of designing a program that will face only decision problem $\mathrm{SAO}_d$ with probability $1$. Here, $\mathrm{SAO}_d=\{ a_0,a_1 \}$ is the decision problem where $a_0$ is known to pay off $\nicefrac{1}{2}$ and $a_1$ is $1$ if $d$ selects a program that selects $a_0$ in $\mathrm{SAO}_d$ and $0$ otherwise. Then $d$ cannot select the optimal agent program for this problem.}

\extendedonlybit{\subsection{Garrabrant inductors}}
\abridgedonlybit{\textbf{Garrabrant inductors~~}}
\label{sec:rel-work-Garrabrant}
The present is in part inspired by the work of Garrabrant et al.~\citeyear{Garrabrant2016}, who address the problem of assigning probabilities under computational constraints and possibilities of self-reference. %
As an alternative to the present theory of BRIAs, one could also try to develop a theory of boundedly rational choice by maximizing expected utility using the Garrabrant inductor's probability distributions. %
Unfortunately, this approach fails for reasons related to the challenge of making counterfactual claims, as pointed out by Garrabrant \citeyear{Garrabrant2017}. As in the case of Hannan consistency, we can address this problem using randomization over actions. However, like Garrabrant (ibid.), we do not find it satisfactory to \textit{require} randomization (cf.\ again \Cref{appendix:on-randomization-and-regret}). We conjecture that\extendedonlybit{, like regret minimizers,} Garrabrant inductors with (pseudo-)randomization could be used to construct BRIAs.

\section{Conclusion}

We developed BRIA theory as a theory of bounded inductive rationality. We gave results that show the normative appeal of BRIAs. Furthermore, we demonstrated the theory's utility by using it to justify Nash equilibrium play. At the same time, the ideas presented lead to various further research questions, some of which we have noted above. We here give three more that we find particularly interesting. Can we modify the BRIA requirement so that it implies coherence properties à la Garrabrant et al.~\citeyear{Garrabrant2016}? Do the frequencies with which BRIAs play the given pure strategies of a game converge to mixed Nash and correlated equilibria? Can BRIA theory be used to build better real-world systems?

\section*{Acknowledgments}

We thank Emery Cooper, Daniel Demski, Sam Eisenstat, Daniel Kokotajlo, Alexander Oldenziel, Nisan Stiennon, Johannes Treutlein and attendants of OptLearnMAS 2021 for helpful discussions.

\bibliographystyle{eptcsini}
\bibliography{references}

\appendix

\section{Proofs}
\label{appendix:proofs}

\subsection{An easy lemma about test sets}

We start with a simple lemma which we will use to simplify a few of our proofs. Roughly, the lemma shows that to cover a hypothesis $h$, it never helps to test $h$ in rounds in which $h_t=0$, i.e., in rounds in which $h$ doesn't make any promises.

\begin{lemma}\label{lemma:testing-positive-estimate}
Let $\bar h$ be a hypothesis and $N\subseteq \mathbb{N}$ s.t.\ $t\in N$ implies $h^e_t=0$. Then if $\bar\alpha$ covers $\bar h$ with test set $M$, $\bar\alpha$ covers $\bar h$ with test set $M-N$.
\end{lemma}

\begin{proof}
For all $T$, we have that
\begin{eqnarray*}
l_T(\bar\alpha,\bar r, M, \bar h) = \sum_{t\in M_{\leq T}} r_t - h_t^e &=& \sum_{t\in M_{\leq T} - N} r_t - h_t^e + \sum_{t\in M_{\leq T} \cap N} r_t - h_t^e\\
&=& \sum_{t\in M_{\leq T} - N} r_t - h_t^e + \sum_{t\in M_{\leq T} \cap N} r_t\\
&\geq & \sum_{t\in M_{\leq T} - N} r_t - h_t^e\\
&=& l_t(\bar \alpha, \bar r, M-N, \bar h).
\end{eqnarray*}
Thus, if $l_T(\bar\alpha,\bar r, M, \bar h)\rightarrow -\infty$ as $T\rightarrow -\infty$, it must also be $l_T(\bar\alpha,\bar r, M-N, \bar h)\rightarrow -\infty$ as $T\rightarrow -\infty$.
\end{proof}

\subsection{Proof of \Cref{thm:computable-BRIAs}}
\label{appendix:computing-BRIAs}

\computingBRIAsthm*

\begin{proof}
Our proof is divided into four parts. First, we give the generic construction for a BRIA (1). Then we show that this is indeed a BRIA by proving that it satisfies the no overestimation criterion (2), as well as the coverage criterion (3). Finally, we show that under the assumptions stated in the theorem, this BRIA is computable in the claimed time complexity (4).

\underline{1. The construction}

First, we need an \textit{allowance function} $A:\mathbb{N}\times \mathbb{N}\rightarrow \mathbb{R}_{\geq 0}$, which for each time $n$, specifies a positive amount $A(n,i)$ given to hypothesis $h_i$'s wealth at time $n$. The allowance function must satisfy the following requirements:
\begin{itemize}
\item Each hypothesis must get infinite overall allowance, i.e.,
$\sum_{n=1}^{\infty} A(n,i)=\infty$
for all hypotheses $h_i$.
\item The overall allowance distributed per round $n$ must go to zero, i.e.,
\begin{equation}\label{eq:allowance-to-zero}
\sum_{n=1}^N \frac{1}{N} \sum_{i=1}^{\infty} A(n,i) \underset{N\rightarrow \infty}{\rightarrow} 0.
\end{equation}
In particular, the allowance distributed in any particular round must be finite.
\end{itemize}
An example of such a function is $A(n,i) =  n^{-1}i^{-2}$.

We can finally give the algorithm itself. Initialize the wealth variables as (for example) $w_0(i)\leftarrow 0$ for each hypothesis $h_i\in \mathbb{H}$.

At time $t$, we run a (first-price sealed-bid\footnote{This format is mainly chosen for its simplicity. We could just as well use a second-price (or third-price, etc.) auction. We could use even different formats to get somewhat different BRIA-like properties. For instance, with combinatorial auctions, one could achieve cross-decision optimization. %
}) auction for the present decision problem among all hypotheses. That is, we determine a winning hypothesis
\begin{equation}\label{eq:highest-hypothesis}
    i^*_t \in \argmax_{i\in \mathbb{N}} \min (h_{i,t}^e , w_t(i))
\end{equation}
with arbitrary tie breaking. Intuitively, each hypothesis $h_i$ bids $h_{i,t}^e$, except that it is constrained by its wealth $w_t(i)$. The idea is that if $h_i$ has performed poorly relative to its promises, then $\alpha$ should not trust $h_i$'s promise for the present problem.
Let $e^*_t\in[0,1]$ be the maximum (wealth-bounded) bid itself. We then define our agent at time $t$ as $\alpha_t \coloneqq (h_{i_t^*,t}^c , e^*_t)$.

We update the wealth variables as follows. For all hypotheses $i\neq i^*_t$, we merely give allowance, i.e., $w_{t+1}(i) \leftarrow w_t(i) + A(t,i)$.
For the winning hypothesis $i_t^*$, we update wealth according to $w_{t+1}(i_t^*) \leftarrow w_t(i^*_t)
    + A(t,i_t^*)
    + r_t - e^*_t$.
That is, the highest-bidding hypothesis receives the allowance and the reward obtained after following its recommendation ($r_t$), but pays its (wealth-bounded) bid ($e_t^*$).

\underline{2. No overestimation} We will show that the cumulative overestimation is bounded by the sum of the allowance.

For each $T$, let $B^+_{T}$ be the set of hypotheses whose wealth $w_t(i)$ is positive for at least one time $t\in\{0,...,T\}$. Note that all highest-bidding hypotheses in rounds $1....,T$ are in $B^+_{T}$ for all $j$. We can then write the overall wealth of the hypotheses in $B^+_{T}$ at time $T$ as
\begin{equation*}
\sum_{i\in B^+_{T}} w_T (i) = \sum_{i\in B^+_{T}} \sum_{n=1}^{T} A(n,i) + \sum_{t=1}^T r_t-\alpha_t^e.
\end{equation*}
That is, the overall wealth at time $T$ is the allowance distributed at times $1,...,T$ plus the money earned/lost by the highest-bidding hypotheses.

Now notice that by the construction above, if a wealth variable $w_t(i)$ is non-negative once, it remains non-negative for all future $t$. Thus, for all $i\in B_{T}^+$, $w_T(i)\geq 0$. 
Second, the last term is the negated cumulative overestimation of $\bar \alpha$. Thus, re-arranging these terms and dividing by $T$ gives us the following upper bound on the per-round overestimation:
\begin{equation*}
\frac{1}{T} \mathcal{L}_T(\alpha,\bar r) = \frac{1}{T}\left(\sum_{i\in B_{T}^+} \sum_{n=1}^{T} A(n,i) - \sum_{i\in B^+_{T}} w_T (i) \right) \leq \frac{1}{T}\sum_{i\in B_{T}^+} \sum_{n=1}^{T} A(n,i) \leq \sum_{i=1}^{\infty} \frac{1}{T} \sum_{n=1}^{T} A(n,i),
\end{equation*}
which goes to zero as $T\rightarrow \infty$ by our requirement on the function $A$ (line \ref{eq:allowance-to-zero}).\\\\

\underline{3. Hypothesis coverage} Given a hypothesis $h_i$ that strictly outpromises $\bar \alpha$ infinitely often, we use as a test $M_i$, the set of times $t$ at which $h_i$ is the winning hypothesis (i.e., the set of times $t$ s.t.\ $i=i_t^*$). We have to show that $M_i$ is infinite, is a valid test set (as per \Cref{def:test-set}), and that it satisfies the justified rejection requirement in the hypothesis coverage criterion.

A) We show that $M_i$ is infinite. That is, we need to show that infinitely often $h_i$ is the highest-bidding hypothesis in the auction that computes $\bar\alpha$. Assume for contradiction that $M_i$ is finite. We will show that at some point $h_i$'s bidding in the construction of $\bar \alpha$ will not be constrained anymore by $h$'s wealth. We will then find a contradiction with the assumption that $h_i$ strictly outpromises $\alpha$ infinitely often.

Consider that for $T'>T$, it is $w_{T'}(i) = w_T(i) + \sum_{t=T+1}^{T'} A(t, i)$. That is, from time $T$ to any time $T'$, hypothesis $i$'s wealth only changes by $h_i$ receiving allowance, because $i$ is (by assumption) not the winning hypothesis $i^*_t$ in any round $t\geq T$. Because we required $\sum_{n=1}^\infty A(n,i)=\infty$, we can select a time $T*\geq T$ such that $w_{T*}(i)\geq 1$. Note that again it is also for all $t>T*$ the case that $w_{t}(i)\geq 1$.
    
We now see that if $t\geq T*$ the wealth constraints is not restrictive. That is, for all such $t$ it is $\min (h_{i,t}^e , w_t(i))=h_{i,t}^e$.
But it is infinitely often $h_{i,t}^e>\alpha_t^e$. This contradicts the fact that by construction, $\alpha_t$ is equal to the highest wealth-restricted hypothesis.
 
B) The fact that $M_i$ is a valid test set follows immediately from the construction -- $\alpha$ always chooses the recommendation of the highest-bidding hypothesis. 
    
C) We come to the justification part of the coverage criterion. Let $B_i$ be the set of rounds in which $\bar h _i$ strictly outpromises $\bar \alpha$. %

At each time $t\in B_i$, by construction $w_T(i,j)<h_{i,t}^e(\DP_T)$.
We have that $h_{i,t}^e(\DP_T)\leq 1$ and
\begin{equation*}
    w_T(i) = \sum_{n=1}^{T} A(n, i) +  \sum_{t\in M_i :t<T} r_t - h_{i,t}^e.
\end{equation*}
Hence, from the fact that $w_T(i)<h_{i,t}^e(\DP_T)$ for all $T\in B_i$, it follows that for all $T\in B_i$, it is
\begin{equation*}
    \sum_{t\in M_i:t<T}  h_{i,t}^e - r_t > \sum_{n=1}^{T} A(n, i),
\end{equation*}
which goes to infinity as $T\rightarrow \infty$, as required.\\

\begin{sloppypar}
\underline{4. Computability and computational complexity} It is left to show that if $\mathbb{H}$ can be computably enumerated and consist only of ($O(g(t))$-)computable hypotheses, then we can implement the above-described BRIA for $\mathbb{H}$, $\overline{\DP}$,$\bar r$ in an algorithm (that runs in $O(g(t)q(t))$, for arbitrarily slow-growing, $O(g(t))$-computable $q$ with $q(t)\rightarrow\infty$).
\end{sloppypar}

The main challenge is that the construction as described above performs at any time $t$, operations for all (potentially infinitely many) hypotheses.
The crucial idea is that for an appropriate choice of $A$, we only need to keep track of a finite set of hypotheses, when calculating $\bar \alpha$ in the first $T$ time steps. Each hypothesis starts with an initial wealth of $0$. Then a hypothesis $i$ can only become relevant at the first time $t$ at which $A(t,i)>0$. At any time $t$, we call such hypotheses \textit{active}. Before that time, we do not need to compute $\bar h_i$ and do not need to update its wealth. By choosing a function $A$ s.t.\ (in addition to the above conditions) $A(t,\cdot)$ has finite, e.c.\ support at each time $t$, we can keep the set of active hypotheses finite at any given time. (An example of such a function is $A(n,i) =  n^{-1}i^{-2}$ for $i<n$ and $A(n,i) =0$ otherwise.) We have thus shown that it is enough to keep track at any given time of only a finite number of hypotheses.

At any time, we therefore only need to keep track of a finite number of wealth variables, only need to compute the recommendations and promises of a finite set of hypotheses, and only need to compute a minimum of a finite set in line \ref{eq:highest-hypothesis}.

Computability is therefore proven. We proceed to show the claim about computational complexity. At any time $t$, let $C_{\max}(t)$ be the largest constant by which the computational complexity of hypotheses at time $t$ are bounded relative to $g(t)$. Further, let $h_b(t)$ be the set of active hypotheses. Then the computational cost from simulating all active hypotheses at time $t$ is at most $h_b(t)C_{\max}(t)g(t)$.
All of $C_{\max}(t)$ and $h_b(t)$ must go to $\infty$ as $t\rightarrow \infty$. However, this can happen arbitrarily slowly, up to the limits of fast ($O(g(t))$) computation.
Hence, if we let $q(t)=h_b(t)C_{\max}(t)g(t)$, we can let $q$ grow arbitrarily slowly (again, up to the limits of fast computation).

Finally, we have to verify that all other calculations can be done in $O(q(t)g(t))$: To determine the winning hypothesis given everyone's promises, we have to calculate the maximum of $h_b(t)\in O(q(t))$ numbers, which can be done in $O(q(t))$ time. We also need to conduct the wealth variable updates themselves, which accounts for $O(h_b(t))$ additions. Again, this is in $O(g(t)q(t))$. And so on.
\end{proof}

\abridgedonlybit{\subsection{Proof of \Cref{thm:no-ec-BRIA}}}
\extendedonlybit{\subsection{Proof of \Cref{thm:no-ec-BRIA}} (and some discussion)}
\label{sec:no-ec-BRIA}

\noecBRIA*

This is shown by a simple diagonalization argument. If a BRIA $\alpha$ were computable (in $O(g(t))$), then consider the hypothesis who in rounds in which $|\DP_t|\geq 2$ and $\alpha_t^e<1$, promises $1$ and recommends an option other than $\alpha_t^c$; and promises $0$ otherwise. This hypothesis strictly outpromises $\alpha$ infinitely often, is computable (in $O(g(t))$) but is never tested
\extendedonlybit{, see Lemma \Cref{lemma:testing-positive-estimate}}.\extendedonlybit{The same argument can similarly be used to show that, for example, a BRIA for the set of polynomial-time hypotheses cannot be computed in polynomial time.}

\begin{extendedonlyblock}
Is diagonalization a silly reason to fail the BRIA criterion? After all, the diagonalizing hypothesis' recommendation will not be sensible in general. For example, imagine that $\DP_t=\{a_{t,0},a_{t,1}\}$ is the problem of guessing the $t$-th binary digit of $\pi$. Option $a_{t,k}$ pays off $1/2$ if the $t$-th digit of $\pi$ is $k$ and $0$ otherwise. This problem can be solved in $O(t)$, so we might like to say that there is an $O(t)$ BRIA for this process. But if we use as $\mathbb{H}$ the set of linear-time-computable hypotheses, this is not possible.

There are two possible responses to this intuition. The first is that we here assume that we at some point come to be certain of how $\bar \DP,\bar r$ work and in particular, that the diagonalizing hypothesis does not work. However, the BRIA criterion for the set of hypotheses $\mathbb{H}$ assumes that we never develop perfect confidence that any of the hypotheses in $\mathbb{H}$ is wrong. If we really wanted to allow an $O(t)$ BRIA, we should therefore exclude the diagonalization hypothesis from $\mathbb{H}$.

A second approach is to avoid diagonalization by randomizing. In particular, we could let the BRIA test hypotheses according to a randomized scheme, where the hypotheses' computational model does not have access to the BRIA's sequence of random variables. This allows us to construct, for example, an $O(t)$ (plus randomization) BRIA that cannot be exploited, even by more powerful hypotheses. However, this, of course, only works if the decision problem sequence is easy to solve: in this case, solvable in $O(t)$, as the following theorem illustrates.

\begin{theorem}
For any $g$ with $g(t)\rightarrow \infty$, there exists a decision problem sequence $\bar \DP,\bar r$ for which there is no BRIA relative to $O(g(t))$ hypotheses that can be computed with the use of randomization in $O(g(t))$.
\end{theorem}
\end{extendedonlyblock}

\subsection{Proof of \Cref{thm:lower-bounds-from-easy-options}}
\label{appendix:proof-of-easy-options-thm}

\easyoptionsthm*

\begin{proof}
We will show that if the assumptions are satisfied, then for all but finitely many $t$, we have that $\alpha_t^e\geq L_t$. From this and the fact that $\bar \alpha$ doesn't overestimate, it then follows that $\sum_{t=1}^T r_t / T\geq \sum_{t=1}^T L_t / T$.

We prove this new claim by proving a contrapositive. In particular, we assume that $\alpha^e_t  < L_t$ for infinitely many $t$ and will then show that $\bar \alpha$ is not a BRIA (using the other assumptions of the theorem).

Consider hypothesis $\bar h_i$ such that $h_{i,t}=(a_t,L_t)$. Because $\bar L$ is e.c.\ and the $\bar a$ are efficiently identifiable, $\bar h$ is e.c. We now show that $\bar h_i$ is not covered by $\bar \alpha$, which shows that $\bar\alpha$ is not a BRIA. By assumption, $\bar h_i$ strictly outpromises $\bar \alpha$ infinitely often. It is left to show that there is no $M_i$ as specified in the hypothesis coverage criterion, i.e.\ no $M_i$ on which $\bar h_i$ consistently underperforms its promises.

If $t\in M_i$, then $\alpha_t^c=h_{i,t}^c=a_t$ and therefore $r_t\geq L_t$. It follows that for all $T$,
\begin{equation*}
    l_T(\bar \alpha,\bar r,M_i,\bar h_i)=\sum_{t\in M_i:t<T}  \underbrace{r_t}_{\geq L_t} - \underbrace{h_{i,t}^e}_{=L_t} \geq 0.
\end{equation*}
Thus, $\bar\alpha$ violates the coverage criterion for $\bar h_i$.
\end{proof}

\subsection{Proof of \Cref{thm:pseudo-lotteries}}

\vMWCdef*

\pseudolotteriesthm*

\begin{proof}
We prove the theorem by proving that for all $\epsilon>0$, $\alpha_t^e\geq \mu_t-\epsilon$ for all but finitely many $t$. As usual, we prove this by proving the following contrapositive: assuming this is not the case, $\bar\alpha$ is not a BRIA. To prove this, consider hypothesis $\bar h_{a,\epsilon}$ that at each time $t$ promises $\max(\mu_t-\epsilon,0)$ and recommends $a_t$. Since $\bar h_{a,\epsilon}$ infinitely often outpromises $\bar \alpha$, it must tested infinitely often. Let the test set be some infinite set $M\subseteq \mathbb{N}$. By \Cref{lemma:testing-positive-estimate}, we can assume WLOG that for all $t\in M$, $h_{a,\epsilon}^e=\mu_t-\epsilon$.

Now notice that $M$ is by assumption computable in $O(h(t))$ given the information available at time $t$. Now
\begin{equation*}
\frac{1}{|M_{i,\leq T}|} l_T(\alpha,\bar r,M_i,\bar h_i)= \frac{1}{|M_{i,\leq T}|} \sum_{t\in M_{i,\leq T}} r_t-(\mu_t-\epsilon) \underset{\text{w.p. }1}{\rightarrow} \epsilon \text{  as }T\rightarrow \infty,
\end{equation*}
where the final step is by the fact that among rounds where $\alpha_t^c=a_t$, $\bar r$ is vMWC random with means $\bar\mu$. Hence, $\bar h_{a,\epsilon}$'s record  $l_T(\alpha,\bar r,M_i,\bar h_i)$ must be positive in all but finitely many rounds. Thus, $\bar\alpha$'s infinitely many rejections of $\bar h_{a,\epsilon}$ violate the coverage criterion.
\end{proof}

\subsection{Proof of \Cref{thm:folk-theorem}}
\label{appendix:proof-of-folk-theorem}

\begin{lemma}[Minimax Theorem \cite{v1928theorie}]
Let $(A_1,A_2,u_1,u_2)$ be any game. Then   
    \begin{equation*}
        \max_{\sigma_i\in\Delta(A_i)} \min_{a_{-i}\in A_{-i}} u_{i} (\sigma_i,\sigma_{-i}) = \min_{\sigma_{-i}\in \Delta(A_{-i})} \max_{a_i\in A_i} u_{i} (\sigma_i,\sigma_{-i}).
    \end{equation*}
\end{lemma}

\FolkTheorem*

\begin{proof}
    The latter part (\enquote{Conversely,...}) follows directly from \Cref{thm:lower-bounds-from-easy-options}. It is left to prove the existence claim.

    We construct the BRIAs as follows. First we fix positive probabilities $p_{\mathbf{c}}\in (0,1)$ and $(p_{a_i})_{a_i\in A_i}$ for $i=1,2$ (WLOG assume $A_1$ and $A_2$ are disjoint) s.t.\ $p_{\mathbf c}+\sum_{i=1}^2 \sum_{a_i\in A_i} p_{a_i} = 1$. Further let $v_i$ be some number that is strictly greater than Player $i$'s maximin value but strictly smaller than $p_c u_i(\mathbf{c})$. By the assumption that $\mathbf{c}$ is strictly individually rational, such a number exists if we make $p_c$ large enough. Then let $\alpha_{i,t}^e=v_i$ for all $t$. Then in each step the BRIAs jointly randomize\footnote{We here use true randomization for simplicity. The same can be achieved using algorithmic randomness.} independently from all bidders in $\mathbb{H}_1,\mathbb{H}_2$ as follows:
    \begin{itemize}
        \item With probability $p_{\mathbf{c}}$ both players play according to $\mathbf{c}$ by jointly implementing $\mathbf{c}$, e.g., by deterministically cycling through the different strategies in the appropriate numbers.%
        Further, $\alpha_{i,t}^e=v_i$. No hypotheses are tested.
        \item With probability $p_{a_i}$, Player $i$ plays $a_i$ and Player $-i$ plays from $\argmin_{a_{-i}\in A_{-i}} u_i(a_i,a_{-i})$. Player $-i$ estimates $v_{-i}$ and does not test any hypothesis. Player $i$ estimates $v_i$ and tests every hypothesis that estimates more than $v_i$.
    \end{itemize}
We now prove that $\bar \alpha_1,\bar \alpha_2$ thus constructed are BRIAs.

\underline{No overestimation}: 
\begin{equation*}
    \mathcal{L}_T(\bar \alpha_i,\bar r_i)/T = \sum_{t=1}^T (\alpha_{i,t}^e - r_{i,t})/T = \sum_{t=1}^T (v_i - r_{i,t})/T \leq v_i - u_i(\mathbf{c}) \text { as } T\rightarrow \infty.
\end{equation*}
By construction, $v_i - u_i(\mathbf{c})\leq 0$.

\underline{Coverage}:
Let $\bar h_i$ be a hypothesis that outbids $\bar\alpha_i$ infinitely often. Then in particular $\bar h_i$ outbids infinitely often in rounds in which $\bar h_i$ recommends some $a_i$ and $\alpha_{i,t}^c=a_i$. Thus, $\bar h_i$ has an infinite test set $M$ on which the hypothesis' empirical record is
\begin{equation*}
    l_T(\bar\alpha_i,\bar r_i, M, \bar h_i) = \sum_{t\in M_{\leq T}} r_t - h_{i,t}^e = \sum_{t\in M_{\leq T}} \min_{a_{-i}} u_i(h_{i,t}^c,a_{-i}) - h_{i,t}^e \leq \sum_{t\in M_{\leq T}} \max_{a_i}\min_{a_{-i}} u_i(a_i,a_{-i}) - v_i \rightarrow -\infty 
\end{equation*}
as $T\rightarrow \infty$.
Thus, $\bar h_i$ is covered.
\end{proof}

\section{Options with random payoffs}
\label{appendix:lotteries-thm}

\label{sec:true-randomness}
The following result shows, roughly, that \extendedonlybit{in the limit BRIAs are von Neumann--Morgenstern rational if von Neumann--Morgenstern rational choice is e.c. That is, }when choosing between different lotteries whose expected utilities are efficiently computable, BRIAs converge to choosing the lottery with the highest expected utility. When other, non-lottery options are available, BRIAs converge to performing at least as well as the best lottery option.

\begin{restatable}{theorem}{lotteriesthm}\label{thm:lower-bounds-from-random-variables}
Let $\bar \alpha$ be a BRIA for $\overline \DP,\bar r$. Let $\bar a$ be a sequence of terms in $\mathcal{T}$ s.t.\ $a_t\in \DP_t$ for all $t\in \mathbb{N}$ and the values of $r_t$ if $\alpha^c_t=a_t$ are drawn independently from distributions with e.c.\ means $\bar \mu$.
Let the $a_t$ be efficiently identifiable from $\DP_t$. Then almost surely in the limit as $T\rightarrow\infty$, it holds that $\sum_{t=1}^T r_r/T \geq \sum_{t=1}^T \mu_t/T$.
\end{restatable}

The proof idea similar to the proof idea for \Cref{thm:lower-bounds-from-easy-options}. It works by considering hypotheses that recommend $a_t$ and promise $\mu_t-\epsilon$ and noting that the empirical record of such hypotheses goes to $-\infty$ with probability $0$.

\begin{proof}
We need only show that with probability $1$ for all $\epsilon>0$ it holds that for all but finitely many times $t$ that $\alpha_t^e\geq \mu_t-\epsilon$. From this and the no overestimation property of $\bar \alpha$, the conclusion of the theorem follow.

Again we prove the following contrapositive: If there is some $\epsilon>0$ s.t.\ with some positive probability $p>0$ we infinitely often have that $\alpha_t^e< \mu_t-\epsilon$, then $\bar\alpha$ is with positive probability not a BRIA.

Consider the hypothesis $\bar h_{a,\epsilon}$ that at each time $t$ promises $\max(\mu_t-\epsilon,0)$ and recommends $a_t$. Since with probability $p$, $\bar h_{a,\epsilon}$ infinitely often outpromises $\bar \alpha$, it must in these cases (and therefore with probability (at least) $p$) be tested infinitely often. (If not, we $\bar\alpha$ would in these cases not be a BRIA and we would be done.) In these cases (i.e., when $\bar h_{a,\epsilon}$ is tested infinitely often), let the test set be some infinite set $M\subseteq \mathbb{N}$. (Note that $M$ may depend on $\bar r$ and inherit its stochasticity. This will not matter for the following, though.) For simplicity, let $M$ be the empty set if $\bar h_{a,\epsilon}$ does not outpromise $\alpha$ infinitely often. By \Cref{lemma:testing-positive-estimate}, we can assume WLOG that for all $t\in M$, $h_{a,\epsilon}^e=\mu_t-\epsilon$. Now notice that
\begin{equation*}
\frac{1}{|M_{i,\leq T}|} l_T(\alpha,\bar r,M_i,\bar h_i)= \frac{1}{|M_{i,\leq T}|}\sum_{t\in M_{i,\leq T}} r_t-h_{a,\epsilon,t}^e = \frac{1}{|M_{i,\leq T}|} \sum_{t\in M_{i,\leq T}} r_t-(\mu_t-\epsilon).
\end{equation*}
Conditioning on the (probability $p$) event that $h$ infinitely often outbids and therefore that $M$ is infinite, it must then with probability $1$ be the case that $\sum_{t\in M_{i,\leq T}} r_t-(\mu_t-\epsilon) \underset{\text{w.p. }1}{\rightarrow} \epsilon$ as $T\rightarrow \infty$ by the law of large numbers. We have thus shown that with positive probability ($p$) $\bar h_{a,\epsilon}$ outpromises $\bar \alpha$ infinitely often while $\bar h_{a,\epsilon}$'s record  $l_T(\alpha,\bar r,M_i,\bar h_i)$ is positive in all but finitely many rounds. Thus, in this positive-probability event $\bar\alpha$'s infinitely many rejections of $\bar h_{a,\epsilon}$ violates the coverage criterion.
\end{proof}

\section{More on randomization and regret}
\label{appendix:on-randomization-and-regret}

In the literature on multi-armed bandit problems, authors usually consider the goal of regret minimization. A natural rationality requirement is for per-round average regret to go to $0$. This is sometimes called Hannan consistency.  For any given agent $c$, the Simplified Adversarial Offer $\mathrm{SAO}_c$ of \Cref{sec:log-unc-strategic-interactions} is a problem on which regret is necessarily high. However, if we assume that the agent at time $t$ can randomize in a way that is independent of how the rewards are assigned by $D_t$, it can actually be ensured that per-round regret (relative to any particular hypothesis) goes to 0 (see \Cref{sec:bandits-with-expert-advice}).\extendedonlybit{ In the literature on such Newcomb-like problems (see \Cref{sec:rel-work-dt-Newcomb}), an idea closely related to regret minimization has been discussed under the name ratificationism \citep[see][for an introduction and overview]{Weirich2016}. Ratificationism similarly uses distributions over actions \citep[see, e.g., the formal description by][]{GranCanaria}, though often these are not meant to arise from randomization \citep[e.g.][]{Arntzenius2008}.}

Arguably the assumption that the agent can independently randomize is almost always satisfied for artificial agents in practice. For instance, if an agent wanted to randomize independently, then for an adversary to predict the program's choices, it would not only need to know the program's source code. It would also require (exact) knowledge of the machine state (as used by pseudo-random number generators); as well as the exact content of any stochastic input such as video streams and hardware/true random number generators. Independent randomization might not be realistic for humans (to whom randomization requires some effort), but none of these theories under discussion (the present one, regret minimization, full Bayesian updating, etc.) are directly applicable to humans, anyway.

Nevertheless, we are conceptually bothered by the assumption of independent randomization. It seems desirable for a theory of choice to make as few assumptions as possible about the given decision problems. Moreover, we can imagine situations in which independent randomization is unavailable to a given agent. It seems odd for a theory of learning to be contingent on the fact that such situations are (currently) rare or practically insignificant. A detailed discussion of this philosophical concern is beyond the scope of this paper.\extendedonlybit{\footnote{For brief discussions of this and closely related concerns in the literature on Newcomb-like problems, see Richter (\citeyear{Richter1984}), Harper (\citeyear{Harper1986}), Skyrms (\citeyear{Skyrms1986}), Arntzenius (\citeyear[Section 9]{Arntzenius2008}), Levinstein and Soares (\citeyear{Levinstein2020}), and Oesterheld and Conitzer (\citeyear[Section IV.1]{ExtractingMoneyFromCDT}).}} %

In the rest of this section, we discuss the goal of regret minimization under the assumption that algorithms \textit{can} randomize independently of $\bar D$. The problems discussed in this section all involve references to the agent's choice.

\begin{extendedonlyblock}

We argue that regret minimization/ratificationism is undesirable in some decision problems (\Cref{sec:regret-minimization-is-implausible}). We then discuss the relationship between BRIA theory and regret minimization/ratificationism in a particular type of decision problem (\Cref{sec:BRIA-randomization-regret}). As a caveat, note that in many cases -- in particular under suitable independence assumptions between $\bar D$ and the agent -- zero-regret \textit{is} desirable. However, we believe in all of these cases regret minimization satisfies the BRIA criterion.

\end{extendedonlyblock}

\begin{extendedonlyblock}

\subsection{The implausibility of regret minimization as a rationality condition}
\label{sec:regret-minimization-is-implausible}

\end{extendedonlyblock}

\label{Newcombs-problem}
We consider a version of Newcomb's problem (introduced by \cite{Nozick1969}; see \Cref{sec:rel-work-dt-Newcomb} for further discussion and references). In particular, we consider for any chooser $c$ the decision problem $\mathrm{NP}_c=\{a_1,a_2\}$ which is resolved as follows. First, we let $D(a_1)=\nicefrac{1}{4}+P(c=a_1)/2$.
So the value of $a_1$ is proportional to the probability that $c$ chooses $a_1$. And second, we let $D(a_2)=D(a_1)+P(c=a_1)/4$.

If we let $p=P(c=a_1)$, then the expected reward of $c$ in this decision problem is $\nicefrac{1}{4}+p/2 + (1-p)p/4$.
It is easy to see that this is strictly increasing in $p$ and therefore maximized if $c=a_1$ deterministically. The regret, on the other hand, of $c$ is $p^2/4$, which is also strictly increasing in $p$ on $[0,1]$ and therefore minimized if $c=a_2$ deterministically. Similarly, the competitive ratio is given by $\frac{1/4+3p/4}{1/4 + p/2+(1-p)p/4}$,
which is also strictly increasing in $p$ on $[0,1]$ and therefore also minimized if $c=a_2$ deterministically.
Regret and competitive ratio minimization as rationality criteria would therefore require choosing the policy that minimizes the actual reward obtained in this scenario, only to minimize the value of actions not taken.

As noted in \Cref{sec:rel-work-dt-Newcomb}, it is a controversial among decision theorists what the rational choice in Newcomb's problem is. However, from the perspective of this paper in this particular version of the problem, it seems undesirable to require reward minimization. Also, it is easy to construct other (perhaps more convincing) cases. For example, if a high reward can be obtained by taking some action with a small probability, then regret minimizers take that action with high probability in a positive-frequency fraction of the rounds. Or consider a version of Newcomb's problem in which $D(a_1)$ is defined as before, but $D(a_2)=D(a_1)$. On such problems, Hannan-consistency is trivially satisfied by any learner, even though taking $a_1$ with probability 1 is clearly optimal.%

\begin{extendedonlyblock}

\subsection{BRIAs, randomization and regret minimization}
\label{sec:BRIA-randomization-regret}

Interestingly, when faced with a sequence of problems like $\mathrm{NP}_c$, the BRIA criterion allows convergence both to $a_1$ and to $a_2$. Which action the BRIA converges to (if any), depends on how it learns, i.e., how it tests hypotheses. If the BRIA chooses deterministically, like the BRIA described in \Cref{appendix:computing-BRIAs}, then $\mathrm{NP}_c$ becomes an easy problem (in the sense of \Cref{sec:easy-options}): whenever $a_1$ is taken, a reward of $3/4$ is obtained; whenever $a_2$ is taken, a reward of $1/2$ is obtained. Hence, by \Cref{thm:lower-bounds-from-easy-options} from that section, deterministic BRIAs must converge to taking $a_1$ with frequency 1.

If a BRIA $\bar \alpha$ randomizes, then the dynamics become much more complicated. When fixing a particular probability of taking each action -- e.g., imagine that the BRIA simply takes each action with probability $1/2$ for a while -- $a_2$ will be associated with higher rewards. However, if the probabilities of taking $a_1$ varies between rounds -- e.g., because at some point the BRIA adjusts its probabilities based on past experience -- then taking $a_1$ is correlated with $P(\alpha^c_t(\mathrm{NP}_{\alpha_t^c}){=}a_1)$ being high, which in turn is correlated with obtaining high rewards.\footnote{The case of a deterministic $\bar \alpha$ can be viewed as one where only the extreme case of this latter phenomenon occurs.} Hence, it is not clear whether $a_1$ or $a_2$ will be empirically associated with higher rewards. Learning processes such as these are analyzed by \citet{GranCanaria}, who show that model-free reinforcement learners can only converge to ratifiable strategies. In this particular problem, this would mean that the learner would converge to taking $a_2$ with probability 1. The case of BRIAs is complicated further by the additional layer of hypotheses. Nonetheless, based on \citeauthor{GranCanaria}'s work we suspect that are there randomizing BRIAs that converge to taking $a_2$ with probability 1 in this problem. More generally:

\begin{conjecture}\label{conjecture:random-BRIA-ratifiable}
There are BRIAs who can only converge to ratifiable/regret-minimizing choice probabilities.
\end{conjecture}

BRIA theory seems to not claim anything about what is the rational choice in this particular decision problem. What one makes of this depends on one's views about the decision theory of Newcomb-like problems. If you are unsure about what to do in Newcomb's problem or believe that both one- and two-boxing can be justified, then it may be an attractive feature of BRIA theory that it admits both possibilities, albeit only in this particular version of Newcomb's problem. If you have a strong opinion about what is the rational choice in this scenario, then you might complain that BRIA theory makes no particular claim about this problem. We discuss the two possible views in turn.

First, if you are a two-boxer, you might find it problematic that BRIA theory allows agents who converge to taking $a_1$ (one-boxing) with frequency 1. The most important point to make about the two-boxer's perspective is that, as noted in \Cref{sec:rel-work-dt-Newcomb}, there are other cases in which BRIA theory sides unambiguously with one-boxers (evidential decision theory), anyway. In these other cases, fundamental changes are necessary to learn to two-box, such as learning a causal model or requiring that appropriate counterfactuals are revealed to the agent. If one believes that in those other scenarios, one-boxing is acceptable, then BRIA theory may serve as a foundation on which one could add additional rationality requirements (e.g., about randomization).

Second, if you are a one-boxer, you might find it problematic that BRIA theory allows agents who converge to taking $a_2$ with frequency (or probability) 1. This complaint is more critical to BRIA theory, because BRIA is, generally speaking, an evidentialist theory. Again, we could view BRIA theory as a first step on which to add other requirements (such as representing all randomization within the available options) to ensure convergence to taking $a_2$ with frequency 1. But we also believe that it is useful to understand what exactly goes wrong in the BRIAs of \Cref{conjecture:random-BRIA-ratifiable}. Roughly, it seems that these BRIAs do not take into account the fact that the way they decide which hypothesis to follow in a particular round affects what reward they get. To converge to optimal behavior (taking $a_1$ with probability 1), these BRIAs would have to keep track not only of how well different hypotheses perform. They would also have to test different aspects of the choice process that tests hypotheses.\footnote{Perhaps in other decision examples, it is explicitly rewarded to randomize. It might be even be rewarded to randomize in a way that is not represented in the chosen option.} It is not clear to us whether the choice probabilities are special among the agent's \enquote{internals}. One could similarly imagine that the reward that an agent obtains to be affected by other internals. For example, imagine a case in which the reward is high if and only if the expert currently tests a hypothesis with a poor test record on its test set. The optimal algorithm would then only test hypotheses that happen to be right but that have a poor record. But it seems too much to ask that an algorithm optimizes all of its internals.

\end{extendedonlyblock}

\begin{extendedonlyblock}
\section{Some regret minimizers satisfy a generalized BRIA criterion}
\label{appendix:regret-minimizers-satisfy-BRIA}

We here show that some regret minimizers satisfy a slightly generalized version of the BRIA criterion. We first have to give a formal definition of regret. Since the literature on adversarial bandit problems with expert advice does not consider experts who submit estimates in the way that our hypotheses do, we cannot use an existing definition and will instead make up our own. For simplicity, we will only consider the case $\mathbb{S}=\mathbb{N}$.

As noted in \Cref{sec:rel-work-contextual-bandits}, to define regret we need counterfactuals. Therefore, throughout this section we assume that instead of selecting a single value $r_t$, the environment selects a function $D_t\colon \DP_t \rightarrow [0,1]$ that assigns a reward even to counterfactual actions. Call such $\bar D$ an \textit{extended} decision process. Instead of $r_t$, we can then write $D_t(\alpha_t^c)$.

Let $\bar D$ be an extended decision process, $\bar\alpha$ be an agent and $\mathbb{H}=\{h_1,h_2,...\}$ be a set of hypotheses. For simplicity, let $\mathbb{H}$ be finite. For each $h_i\in \mathbb{H}$, let $B_i\coloneqq \{ t\in \mathbb{N}\mid h_{i,t}^e>\alpha_t^e  \}$ be the set of rounds in which $h_i$ outpromises $\alpha$. We define the average per-round regret of the learner to hypothesis $h_i$ up to time $T$ as
\begin{equation*}
\mathrm{REGRET}_{m,T}=\mathbb{E}\left[\frac{1}{|B_{m,\leq T}|} \sum_{t\in B_{m,\leq T}} D_t(h_{i,t}^c) - \alpha_t^e \right].
\end{equation*}
As before, the bidding mechanisms means that hypotheses can specialize on specific types of decisions.\footnote{Note that we subtract the agent's \textit{estimates}, not the utility that $\bar \alpha$ in fact achieves. This is important. Otherwise, the learner can set $\alpha^e=0$ even in rounds in which $D_t(\alpha_t^c)$ is (expected to be) high, thus circumventing the expert's bidding mechanism.

Still, there are alternative definitions that also work. For example, one might count regret only in rounds in which $\alpha$ and $h_i$ differ in their recommendations.}
As is common in the adversarial bandit problem literature, we will be interested in learning algorithms that guarantee average regret to go to zero as $|B_{m,\leq T}| \rightarrow \infty$.

Regret is somewhat analogous to the cumulative empirical record on the test set. As with the coverage condition, low regret can be achieved trivially by setting $\alpha^e=1$. Thus, if we replace the coverage criterion with a sublinear-regret requirement, we have to keep the no overestimation criterion.

\begin{conjecture}
Let $\bar D$ be an extended decision process where $|\mathrm{DP}_t|$ is bounded for all $t\in \mathrm{N}$.
With access to an independent source of randomization, and given access to the outputs of all hypotheses in $\mathbb{H}$, we can compute $\bar \alpha$ that does not overestimate on $\mathbb{N}$ s.t.\ for all hypotheses $h_i$, $\mathrm{REGRET}_{i,T}\rightarrow 0 $ with probability $1$ if $|B_{i,\leq T}| \rightarrow \infty$.
\end{conjecture}

As noted elsewhere, without independent randomization it is clear that such an $\bar\alpha$ cannot be designed. Even with independent randomization, it is not obvious whether the conjecture holds. However, similar results in the literature on adversarial bandit problems with expert advice lead us to believe that it does. That said, we have not been able to prove the conjecture by using simply the results from that literature.

\begin{theorem}
Let $\bar\alpha$ be an independently randomized agent that does not overestimate on $\mathbb{N}$ and ensures sublinear regret with probability $1$ relative to all hypotheses in some finite set $\mathbb{H}=\{h_i\}_i$. Further assume that for all hypotheses $h_i$, $P(\alpha_t^c{=}h_{i,t}^c)\in \omega(1/t)$ among $t\in B_i$. Then we can compute based on $\alpha$ a new agent $\tilde \alpha$ that does not overestimate and that satisfies for each hypothesis $h_i$ that is infinitely often rejected,
\begin{equation}\label{line:weighted-relative-loss-neg-inf}
    \sum_{t\in B_{i,\leq T}} \frac{\mathbbm{1}[\alpha_{t}^c{=}h_{i,t}^c]}{P(\alpha_{t}^c{=}h_{i,t}^c)} (D_t(h_{i,t}^c)-h_{i,t}^e) \rightarrow -\infty
\end{equation}
among $T$ at which $\tilde \alpha$ rejects $h_i$.
\end{theorem}

Notice that the left-hand side of line \ref{line:weighted-relative-loss-neg-inf} is a weighted version of the cumulative empirical record on the set $\{ t\in B_{i,\leq T} \mid \alpha_{t}^c{=}h_{i,t}^c \}$.

The proof combines one key idea from the literature on adversarial multi-armed bandits -- importance-weighted estimation -- and one from this paper -- the decision auction construction (\Cref{appendix:computing-BRIAs}).

\begin{proof}
For $t\in B_i$, define
\begin{equation*}
\hat{\mathrm{R}}_{i,t} = \frac{\mathbbm{1}[\alpha_{t}^c{=}h_{i,t}^c]}{P(\alpha_{t}^c{=}h_{i,t}^c)}(D_t(h^c_{i,t}) - \alpha^e_t),
\end{equation*}
where we assume $P(\alpha_{t}^c{=}h_{i,t}^c)>0$.
As usual we then have that $\mathbb{E}\left[ \hat{\mathrm{R}}_{i,t} \right] = D_t(h^c_{i,t}) - \alpha^e_t$.
For $t\notin B_i$, define $\hat{\mathrm{R}}_{i,t}=0$.
Hence, $\hat{\mathrm{R}}_{i,t}$ can be used as an unbiased estimator of the regret in a single round.
Further, $\Var(\hat{\mathrm{R}}_{m,t})\in o(t)$, and thus $\sum_{t=1}^T \Var(\hat{\mathrm{R}}_{m,t}) \in o(T^2)$. By Kolmogorov's strong law of large numbers,
\begin{equation*}
\frac{1}{T} \sum_{t\in B_{m,\leq T}} \hat{\mathrm{R}}_{m,t} - \frac{1}{T}\sum_{t\in B_{m,\leq T}} D_t(h^c_{i,t}) - \alpha^e_t%
\frac{1}{T} \sum_{t=1}^T \hat{\mathrm{R}}_{m,t} - \frac{1}{T} \sum_{t=1}^T  D_t(h^c_{i,t}) - \alpha^e_t%
0 \text{ as  }T\rightarrow \infty
\end{equation*}
In other terms,
\begin{equation*}
 \sum_{t\in B_{m,\leq T}} \hat{\mathrm{R}}_{m,t} - \sum_{t\in B_{m,\leq T}} D_t(h^c_{i,t}) - \alpha^e_t
\end{equation*}
is sublinear.

We now construct new estimates. Fix a non-decreasing, sublinear function $\mathrm{CA}\colon \mathbb{N}\rightarrow \mathbb{R}$ with $\mathrm{CA}(T) \rightarrow \infty$. (These are cumulative versions of the allowance functions from the construction in \Cref{appendix:computing-BRIAs}.)
Next, we define
\begin{equation*}
\mathcal{L}_{i,T}\coloneqq %
\sum_{t\in M_{i,\leq T}} \frac{\mathbbm{1}[\alpha_{t}^c{=}h_{i,t}^c]}{P(\alpha_{t}^c{=}h_{i,t}^c)}(D_t(h^c_{i,t}) - h^e_{i,t})%
\sum_{t\in B_{i,\leq T}-M_i} \frac{\mathbbm{1}[\alpha_{t}^c{=}h_{i,t}^c]}{P(\alpha_{t}^c{=}h_{i,t}^c)}(D_t(h^c_{i,t}) - \alpha^e_t),
\end{equation*}
where $M_{i}\subseteq B_i$ will be defined in a second. Define $w_{T}(i)=\mathrm{CA}(T) + \mathcal{L}_{i,T}$.
Now at each time $t$, we define our new estimate as
\begin{equation}\label{eq:max-hypothesis}
\tilde \alpha_t^e = \max (\alpha_t^e, \max_{i:w_{t-1}(i) \geq 0} h_{i,t}^e).
\end{equation}
Finally, let $M_i$ be the set of rounds in which $i$ is the maximizer in Eq.\ \ref{eq:max-hypothesis} through the outer max.

We now need to show two things: That cumulative overestimation is still sublinear even for the new increased $\tilde \alpha_t^e$ and that the claimed variant of the hypothesis coverage criterion is satisfied.

We start with hypothesis coverage. First notice that because $M_i\subseteq B_i$ and for $t\in B_i$, $h_{i,t}^e> \alpha_{t}^e$, we get that
\begin{equation*}
w_{T}(i) \geq  \mathrm{CA}(T) +  \sum_{t\in B_{m,\leq T}} \frac{\mathbbm{1}\left[\alpha_t^c=h_{i,t}^c\right]}{ P_t(\alpha_t^c=h_{i,t}^c)} (D_t(h^c_{i,t}) - h_{i,t}^e).
\end{equation*}
Thus, whenever $h_{i,T}^e>\tilde \alpha_T^e$, then by construction $w_t(i)< 0$, and therefore
\begin{equation*}
\sum_{t\in B_{i,\leq T}} \frac{\mathbbm{1}\left[\alpha_t^c=h_{i,t}^c\right]}{ P(\alpha_t^c=h_{i,t}^c)} (D_t(h^c_{i,t}) - h_{i,t}^e) \leq -\mathrm{CA}(T).
\end{equation*}
Thus, we get that among $T\in \tilde B_{i}$ (the times where $t$ strictly outpromises the new estimates), the empirical record on the test set goes to $-\infty$.

It is left to show that overestimation remains low if we increase the estimates from $\alpha^e$ to $\tilde \alpha^e$. We have
\begin{equation*}
\sum_{t=1}^T \tilde \alpha_t^e - D_t(h_{i,t}^c) = \sum_{t=1}^T \alpha_t^e - D_t(h_{i,t}^c)+ \sum_{t=1}^T \tilde \alpha_t^e - \alpha_t^e .
\end{equation*}
The first sum is sublinear by assumption. So we only have to show that $\sum_{t=1}^T \tilde \alpha_t^e - \alpha_t^e$ is sublinear in $T$.
We have
\begin{equation}\label{line:increase-of-estimates-on-behalf-of-m}
\sum_{t=1}^T \tilde \alpha_t^e - \alpha_t^e = \sum_i \sum_{t\in M_{i,\leq T}}  h_{i,t}^e - \alpha_t^e.
\end{equation}
So, it is left to show that the increase on behalf of each expert $i$ is sublinear.

Now, we use IWE again. That is, we consider
\begin{equation*}
\sum_{t\in M_{i,\leq T}} \frac{\mathbbm{1}\left[\alpha_t^c=h_{i,t}^c\right]}{ P(\alpha_t^c=h_{i,t}^c)} (h_{i,t}^e - \alpha_t^e).
\end{equation*}
By the same argument as above, we can show that the difference between this term and $\sum_{t\in M_{i,\leq T}}  h_{i,t}^e - \alpha_t^e$ is sublinear. So it is enough to show that this term is sublinear.

Now notice that
\begin{eqnarray*}
w_{T}(i) &=& \mathrm{CA}(T) + \sum_{t\in M_{i,\leq T}} \frac{\mathbbm{1}\left[\alpha_t^c=h_{i,t}^c\right]}{P(\alpha_t^c=h_{i,t}^c)} \underbrace{(D_t(h_{i,t}^c) - h_{i,t}^e)}_{=(D_t(h_{i,t}^c)-\alpha_t^e)-(h_{i,t}^e-\alpha_t^e)} %
+   \sum_{t\in B_{i,\leq T}-M_i} \frac{\mathbbm{1}\left[\alpha_t^c=h_{i,t}^c\right]}{\alpha_t^c=h_{i,t}^c} \left(D_t(h_{i,t}^c) - \alpha_t^e\right)\\
&=&
\mathrm{CA}(T) -  \sum_{t\in M_{i,\leq T}} \frac{\mathbbm{1}\left[\alpha_t^c=h_{i,t}^c\right]}{P_t(\alpha_t^c=h_{i,t}^c)} (h_{i,t}^e - \alpha_t^e) %
+ \sum_{t\in B_{i,\leq T}} \frac{\mathbbm{1}\left[\alpha_t^c=h_{i,t}^c\right]}{P_t(\alpha_t^c=h_{i,t}^c)} (D_t(h_{i,t}^c)-\alpha_t^e).
\end{eqnarray*}
Now, for $T\in M_i$, it must be $w_{T}(i)\geq 0$. Still, $w_{T}(i)$ can fall under $0$, but only by $\hat R^m_t$ for some $t\in \{1,...,T\}$, which is in  $o(T)$. Thus,
\begin{equation*}
\sum_{t\in M_{i,\leq T}} \frac{\mathbbm{1}\left[\alpha_t^c=h_{i,t}^c\right]}{P_t(\alpha_t^c=h_{i,t}^c)} (h_{i,t}^e - \alpha_t^e) %
\mathrm{CA}(T) + \sum_{t\in B_{i,\leq T}} \frac{\mathbbm{1}\left[\alpha_t^c=h_{i,t}^c\right]}{P_t(\alpha_t^c=h_{i,t}^c)} (D_t(h_{i,t}^c)-\alpha_t^e) + o(T)
\end{equation*}
$\mathrm{CA}$ is sublinear by construction and the second summand has been shown to be sublinear above.
\end{proof}
\end{extendedonlyblock}

\begin{extendedonlyblock}
\section{Dominance?}

In this section, we we show that BRIAs do not in general satisfy the dominance criterion. To even formulate the dominance criterion, we have to consider \textit{extended decision processes} as defined in \Cref{appendix:regret-minimizers-satisfy-BRIA}.

\begin{proposition}
There is an extended decision process $\bar D$, a BRIA $\bar\alpha$ for the set of e.c.\ hypotheses and a positive number $\Delta>0$ s.t.\ for all $t\in \mathbb{N}$, $a_t,b_t \in \DP_t$ $D_t(a_t)>D_t(b_t)+\Delta$ but with limit frequency $1$ we have that $\alpha_t^c = b_t$.
\end{proposition}

This is shown by Newcomb's problem (\Cref{Newcombs-problem}).
In fact, Newcomb's problem shows that for any algorithm that constructs BRIAs, there is a $\bar D$ s.t.\ the algorithm's BRIA converges to $a_t$.

Of course, various dominance-like results follow from the results of \Cref{sec:lower-bounds-on-avg-reward}. However, more interesting applications of dominance are arguably ones where the conditions of these results aren't satisfied, e.g., where it is very unclear how one would assign expected utilities to different options. We will now give some reasons for why it's difficult to give any dominance result for BRIAs that does not follow from the results of \Cref{sec:lower-bounds-on-avg-reward}.

The first thing to notice is that relationships such as $D_t(a_t)>D_t(b_t)+\Delta$ (for all $t$) are irrelevant for our theory, as shown by Newcomb's problem, SAO, etc. Instead, our dominance relation needs to be statistical and relative to the test set. Roughly, we must make an assumption that when testing $a_t$, the rewards are (on average) higher (by $\Delta$) than the reward of taking $b_t$ in rounds in which $b_t$ is taken. Of course, this already means that the result will be quite different from traditional notions of dominance.

A second, subtler issue relates to the use of estimates in our theory.\footnote{As noted in \Cref{appendix:simpler-LoE}, an alternative theory could simply require that an agent tests various choice policies and in the limit follows the ones that are empirically most successful. For such a theory, a condition like the one in the previous paragraph probably suffices.} To ensure that $b_t$ is not taken with limit frequency, we would need to ensure not only that the $a_t$-recommending hypothesis doesn't underperform on its test set (as described above). We also need to ensure that this hypothesis is tested on a set on which it doesn't overestimate. We therefore need a further assumption that gives us some way to safely and efficiently estimate $a_t$, e.g., based on past values of $a_t,b_t$ or estimates $\alpha^e_t$. While this assumption can be made in relatively sneaky ways, we have not found any particularly interesting version of this claim.

We now discuss a subtler issue that relates to the use of estimates in our theory to show why a particularly simple approach doesn't work. A first attempt might be to assume that for every test set $M$, $\avg_{t\in M_{\leq T}} D_t(a_t) > \Delta + \avg_{t \leq T \colon \alpha_t^c = b_t} \alpha_t^c$ as $T\rightarrow \infty$, where $\avg_{t\in N} f(t) \coloneqq \nicefrac{1}{|N|} \sum_{t\in N} f(t)$ for any finite set $N$ and function $f$ on $N$. That is, we assume that $a_t$ performs better on any test set than $b_t$ when taken by $\alpha$. 
The trouble is that to obtain a conclusion we need to transform such an assumption into a hypothesis that not only recommends $a_t$ (and thus receives relatively high rewards on average) but also makes appropriate estimates. 

\end{extendedonlyblock}

\section{Why an even simpler theory fails and estimates are necessary}
\label{appendix:simpler-LoE}

A simple mechanism of learning to choose is the \textit{law of effect} (LoE)\extendedonlybit{ (\citealp{Thorndike1911}, p.\ 244)}:
\begin{quote}
    Of several responses made to the same situation, those which are accompanied or closely followed by satisfaction to the animal will, other things being equal, be more firmly connected with the situation, so that, when it recurs, they will be more likely to recur; those which are accompanied or closely followed by discomfort to the animal will, other things being equal, have their connections with that situation weakened, so that, when it recurs, they will be less likely to occur. The greater the satisfaction or discomfort, the greater the strengthening or weakening of the bond.
\end{quote}
This notion is implicit in many reinforcement learning algorithms\extendedonlybit{ \citep[cf.][Sect.~1.6]{Sutton1998}}. In (human) psychology it is also known as operant conditioning.

In situations like ours, where situations generally do not repeat exactly, for the law of effect to be meaningful, we have to applied on a meta level to general hypotheses or policies for making choices. So let a policy be a function that maps observations to actions. Then we could phrase this meta LoE as: if following a particular policy is accompanied with high rewards, then an agent will follow this policy more often in the future.

The BRIA criterion can be seen as abiding by this meta LoE, as the BRIA criterion requires testing different hypotheses and following the ones that have experimentally proven themselves. Its main conceptual innovation relative to the meta LoE is the bidding system, i.e., having the agent as well as hypotheses give estimates for how much utility will be achieved by making a particular choice, and using these estimates for testing and evaluation. A natural question then is: Are these conceptual additions to meta LoE necessary to obtain the kind of results we obtain? We here show why the answer is yes.

The biggest problem is quite simple to understand: if we don't restrict the testing regimen for policies, then biased testing can justify clearly suboptimal behavior. As an illustrative example, imagine that for all $t$, $\mathrm{DP}_t\in \mathrm{Fin}([0,1])$ where $r_t=\alpha^c_t$. That is, at each time the agent is offered to choose from some set of numbers between $0$ and $1$ and then obtains as a reward the chosen number. The agent tests two policies: The first simply chooses the maximum number. The second chooses, e.g., the worst option that is greater than $1/2$ if there is one, and the best option otherwise.

Of course, in this situation one would like the agent to learn at some point to follow the max policy. BRIAs indeed learn this policy (when accompanying the two tested policies with appropriate estimates) (cf.\ \Cref{thm:lower-bounds-from-easy-options}). But now imagine that the agent tests the max hypothesis primarily in rounds where all values are at most $1/2$ and the other hypothesis primarily in rounds in which there are options greater than $1/2$. Then the max hypothesis could empirically be associated with lower rewards than the max hypothesis, simply because it is tested in rounds in which the maximum achievable reward is lower.

To avoid this issue we would have to require that the set of decision problems on which hypothesis A is tested is in all relevant aspects the same as the set of decision problems on which hypothesis B is tested. Unfortunately, we do not know what the \enquote{relevant aspects} are. For instance, in the above problem it may be sufficient to test the max hypothesis on even time steps and the other hypothesis on odd time steps. However, there may also be problems where rewards depend on whether the problem is faced in an even or in an odd time step. More generally, it is easy to show that for each deterministic procedure of deciding which hypothesis to test, there is a decision process $\bar \DP,\bar r$ in which which this testing procedure introduces a relevant bias. In particular, the positive results we have proven in \Cref{thm:lower-bounds-from-easy-options,thm:lower-bounds-from-random-variables,thm:pseudo-lotteries} seem out of reach. We conclude that a direct deterministic implementation of meta LoE (without the use of estimates) is insufficient for constructing a criterion of rational choice. %

Besides the estimates-based approach to this problem that we have developed in this paper, a different (perhaps more obvious) approach to this problem is to test \textit{randomly}. For this, we assume that we have a randomization device available to us that is independent of $\bar \DP,\bar r$. If we then, for example, randomize uniformly between testing two hypotheses, testing is unbiased in the sense that for any potentially property of decision problems, as the number of tests goes to infinity, both hypotheses will be tested on the same fraction of problems with and without that property. This is essentially the idea behind randomized controlled trials. We have discussed this idea in \Cref{appendix:on-randomization-and-regret}.

\begin{extendedonlyblock}
\section{Factoring team decisions}

\begin{theorem}
Let $n\in\mathbb{N}$ be a positive natural number. Let $\overline \DP_t$ be a a decision problem sequence where  every $t\in \mathbb{N}$, $\DP_t=\DP_{t,1} \times ... \times \DP_{t,n}$, for some sets $\DP_{t,1},...,\DP_{t,n}$. Let $\bar \alpha$ be a BRIA for $\overline \DP, \bar r$ covering the set of e.c.\ hypotheses. Now for any $t$ let $((a_{t,1}\in DP_{t,1},...,a_{t,n}\in \DP_{t,n}),v_t)=\alpha_t$ and define $\alpha_{i,t}=(a_{t,i},v_t)$.%
Then for $i=1,..,n$, $\alpha_i$ is a BRIA for $\overline DP_i, \bar r$ covering the e.c.\ hypotheses.
\end{theorem}

\extendedonlybit{Instead of considering sets $\DP_t$ that are already the Cartesian products of a bunch of sets, one could also factorize any given set (unless its number of elements is $1$ or a prime number) \citep[][Sect.~2]{Garrabrant2021}. For example, a decision from $\{1,2,3,4\}$ can be factorized into a decision of $\{1,2\}$ versus $\{3,4\}$, and a decision of $\{1,3\}$ versus $\{2,4\}$.}

\begin{proof}
\underline{Low overestimation}:
Clearly,
\begin{equation*}
    \mathcal{L}(\bar \alpha_i,\bar r_i) = \sum_{t=1}^T \alpha_{i,t}^e - r_t = \sum_{t=1}^T \alpha_t^e - r_t \leq 0,
\end{equation*}
where the last step is by the assumption that $\bar\alpha$ is a BRIA and therefore does not overestimate in the limit.

\underline{Coverage}: Let $\bar h_i$ be an e.c.\ hypothesis for $\overline \DP _i$,$\bar r$. Let $\bar h$ be a hypothesis s.t.\ the $i$-th entry of $h_t^c$ is equal to $h_t^c$, and $h_t^e=h_{i,t}^e$. Clearly, such an e.c.\ hypothesis exists. Let $M$ be $\bar \alpha$'s test set for $\bar h$. We will also use $M$ as $\bar \alpha_i$
s test set for $\bar h_i$. Also, let $B$ be the set of times at which $h_i$ outpromises $\bar \alpha_i$. Note that $B$ is thereby also equal to the set of times at which $\bar h$ outbids $\bar \alpha$.

We now need to show that if $B$ is infinite, then $(l_T(\bar\alpha_i,\bar r, M, \bar h _i))_{T\in B} \rightarrow -\infty$.
To prove this, notice that for all $T$,
\begin{eqnarray*}
l_T(\bar\alpha_i,\bar r, M, \bar h _i) &=& \sum_{t\in M_{\leq T}} r_t-h_{i,t}^e\\
&=& \sum_{t\in M_{\leq T}} r_t-h_{t}^e\\
&=& l_T(\bar\alpha,\bar r, M, \bar h).
\end{eqnarray*}
By assumption that $\bar \alpha$ is a BRIA, the final term goes to $-\infty$ within $T\in B$ if $B$ is infinite.
\end{proof}

So if $\bar\alpha$ is a BRIA, $\bar\alpha_1,...,\bar\alpha_n$ are BRIAs. Note that the converse of this does not hold.

\end{extendedonlyblock}

\section{Schnorr bounded algorithmic randomness}
\label{appendix:schnorr}

\begin{definition}
\begin{sloppypar}
A \textit{martingale} is a function $d\colon \mathbb{B}^* \rightarrow [0,\infty)$ s.t.\ for all $w\in \mathbb{B}^*$ we have that $d(w)= \nicefrac{1}{2}d(w0)+\nicefrac{1}{2}d(w1)$.
Let $w\in \mathbb{B}^\infty$ be an infinite sequence. We say that $d$ succeeds on $w$ if $\limsup_{n\rightarrow \infty} d(w_1...w_n) = \infty$.
\end{sloppypar}
\end{definition}

\begin{definition}
We call $w\in \mathbb{B}^\omega$ \textit{($O(g(t))$-boundedly) Schnorr random} if there is no martingale $d$ such that $d$ succeeds on $w$ and $d$ can be computed (in $O(g(t))$) given everything revealed by time $t$.
\end{definition}

\begin{theorem}\label{thm:Schnorr-random-lotteries}
Let $\alpha$ be an ($O(h(t))$-computable) BRIA for $\overline \DP,\bar r$ covering all e.c.\ hypotheses.
Let $\bar a$ be a sequence of terms in $\mathcal{T}$ s.t.\ $a_t\in \DP_t$ for all $t\in \mathbb{N}$ and the values $r_t$ in the rounds $t$ with $\alpha^c_t$ are ($O(h(t))$-boundedly) Schnorr random. Then in the limit as $T\rightarrow\infty$, it holds that $\sum_{t=1}^T r_t/T \geq \nicefrac{1}{2}$.
\end{theorem}

\begin{proof}
We conduct a proof by proving the following contrapositive: if the conlusion of the theorem does not hold, then $(r_t)_{t:\alpha_t^c=a_t}$ is not Schnorr random.
Assume that there is $\epsilon>0$ s.t.\ $\sum_{t=1}^T r_t/T<\nicefrac{1}{2}-\epsilon$ for infinitely many $T$. Then by the no overestimation criterion, there must also be an $\epsilon>0$ s.t.\ $\sum_{t=1}^T \alpha^e_t/T < \nicefrac{1}{2}-\epsilon$ for infinitely many $T$. Consider the hypothesis $h_{a,\epsilon}$ that always estimates $\nicefrac{1}{2}-\epsilon$ and recommends $a_t$. Now let $M_\epsilon$ be $\bar\alpha$'s test for $h_{a,\epsilon}$. From the fact that $\bar\alpha$ rejects $h_{a,\epsilon}$ infinitely often, it follows that there are infinitely many $T\in\mathbb{N}$ such that $\sum_{t\in M_{\leq T}} r_t - (\nicefrac{1}{2}-\epsilon) < 0$.

\begin{sloppypar}
From this fact, we will now define an ($O(h(t))$-computable) martingale $d$ that succeeds on the sequence $(r_t)_{t:\alpha_t^c=a_t}$. First, define $d()=1$. Whenever $T$ is not in $M$, define $d((r_t)_{t<T:\alpha_t^c=a_t}0)=d((r_t)_{t<T:\alpha_t^c=a_t})=d((r_t)_{t<T:\alpha_t^c=a_t}1)$. That is, when $T\notin M$, don't bet on $r_T$. If $T\in M$, then bet some small, constant fraction $\delta$ of the current money that the next bit will be $0$. That is, $d((r_t)_{t<T:\alpha_t^c=a_t}0)=(1+\delta)d((r_t)_{t<T:\alpha_t^c=a_t})$ and $d((r_t)_{t<T:\alpha_t^c=a_t}1)=(1-\delta)d((r_t)_{t<T:\alpha_t^c=a_t})$. Clearly, $d$ thus defined is a martingale that is computable based on $\bar\alpha,M$.
\end{sloppypar}

Now we now know that there are infinitely many $T$ s.t.\ $d((r_t)_{t<T:\alpha_t^c=a_t})\geq (1+\delta)^{T+\epsilon T}(1-\delta)^{T}$. It is left to show that for small enough $\delta$, $(1+\delta)^{T+\epsilon T}(1-\delta)^{T}\rightarrow \infty$ as $T\rightarrow \infty$.

First notice that
\begin{equation*}
    (1+\delta)^{T+\epsilon T}(1-\delta)^{T} = ((1+\delta)(1-\delta))^T (1+\delta)^{T\epsilon} = (1-\delta^2)^T (1+\delta)^{T\epsilon} = \left((1-\delta^2) (1+\delta)^{\epsilon}\right)^T.
\end{equation*}
So we need only show that for small enough but positive $\delta$, $(1-\delta^2) (1+\delta)^{\epsilon}>1$. The most mechanic way to do this is to take the derivative at $\delta=0$ (where the left-hand side is equal to $1$) and showing that it is positive. The derivative is $\frac{d}{d\delta} (1-\delta^2) (1+\delta)^{\epsilon} = (1+\delta)^\epsilon (\epsilon - \delta(\epsilon+2))$. Inserting $\delta=0$ yields $\epsilon$, which is positive.
\end{proof}

\begin{extendedonlyblock}

\section{A few minor results}

In this section, we give a few minor results about the BRIA criterion. We don't use them anywhere, but they are helpful to understand what the BRIA criterion is about.

First, we simply note that the BRIA criterion becomes (weakly) stronger if we expand the set of hypotheses under consideration, which is immediate from the definitions in \Cref{sec:the-criterion}.

\begin{proposition}
Let $\mathbb{H},\mathbb{H}'$ be sets of hypotheses such that $\mathbb{H}'\subseteq \mathbb{H}$. Then any BRIA for $\mathbb{H}$ is also a BRIA for $\mathbb{H}'$.
\end{proposition}

The following result shows that if we change a BRIA's decisions and estimates for a finite number of decisions in $\overline \DP$, it remains a BRIA.

\begin{proposition}
Let $\bar \alpha$ be a BRIA for $\overline \DP,\bar r$ covering $\mathbb{H}$. If for all but finitely many $t\in \mathbb{N}$ it is $\zeta_t=\alpha_t$, then $\bar \zeta$ is also BRIA for $\mathbb{H}$, $\overline \DP,\bar r$.
\end{proposition}

The following shows that if two hypotheses differ only at finitely many time steps, then coverage of one implies coverage of the other.

\begin{proposition}
Let $\bar\alpha$ be a BRIA covering $h$. Let $h'$ be s.t.\ $h_t=h'_t$ for all but finitely many $t$. Then $\bar\alpha$ covers $h'$.
\end{proposition}

The following states that the decisions can be reordered, as long as this is done by bounded numbers of places, while maintaining the BRIA criterion. Note that, of course, a BRIA's $\alpha_t^c,\alpha_t^e$ are usually calculate as a function of $\DP_1,...,\DP_t$ and $r_1,...,r_{t-1}$. Thus, a reordered BRIA will typically not be computable in this way. 

\begin{proposition}
Let $\bar\alpha$ be a BRIA for $\overline \DP,\bar r$ covering $h$. Let $f\colon\mathbb{N}\rightarrow\mathbb{N}$ be a bijection s.t.\ $f(n)-n$ is bounded (from above and below) (i.e., there is a number $x$ s.t.\ $|f(n)-n|<x$ for all $n\in\mathbb{N}$). Then $\alpha_{f(1)},\alpha_{f(2)},...$ is a BRIA for $D_{f(1)},D_{f(2)},...$ covering $h_{f(1)},h_{f(2)},...$
\end{proposition}

\begin{restatable}{proposition}{increasingestimates}\label{proposition:increasing-estimates-by-a-bit}
Let $\bar \alpha$ be a BRIA for $\overline \DP,\bar r$ covering $\mathbb{H}$. Let $\bar \epsilon$ be a sequence of non-negative numbers such that $\sum_{t=1}^T \epsilon_t/T\rightarrow 0$ as $T\rightarrow \infty$.
Let
$\zeta_t = (\alpha_t^c,\alpha_t^e+\epsilon_t)$ for all $t$. Then $\bar\zeta$ is a BRIA for $\bar \DP,\bar r$ covering $\mathbb{H}$.
\end{restatable}

Note that \textit{de}creasing estimates by a similar sequence $\bar \epsilon$ in general does not maintain the BRIA property. For example, if the estimates in rounds in which an option $``0.5"$ is chosen is decreased below $0.5$, the resulting agent would be exploitable by a hypothesis that recommends $``0.5"$ and promises $0.5$.

\end{extendedonlyblock}

\end{document}